\documentclass[11pt]{article}
\usepackage[margin=1in]{geometry}
\usepackage[T1]{fontenc}
\usepackage{placeins,titlesec}
\titlespacing*{\section}{0pt}{2.4ex plus .4ex minus .2ex}{1.0ex plus .2ex}
\titlespacing*{\subsection}{0pt}{1.8ex plus .3ex minus .2ex}{.8ex plus .1ex}
\usepackage{lmodern,microtype,amsmath,amssymb,amsthm,bm,booktabs,array,enumitem,xcolor,float,tikz,multicol,graphicx}
\usetikzlibrary{positioning,arrows.meta,fit}

\usepackage{etoolbox}
\apptocmd{\thebibliography}{\small\setlength{\itemsep}{3pt}\setlength{\parskip}{0pt}}{}{}
\usepackage[colorlinks=true,allcolors=blue!45!black]{hyperref}
\floatstyle{ruled}\newfloat{algorithm}{tbp}{loa}\floatname{algorithm}{Algorithm}
\newtheorem{theorem}{Theorem}\newtheorem{proposition}{Proposition}\newtheorem{corollary}{Corollary}
\newcommand{\E}{\mathbb E}\newcommand{\R}{\mathbb R}\newcommand{\HH}{\mathcal H}
\newcommand{\A}{\mathcal A}\newcommand{\Sset}{\mathcal S}\newcommand{\norm}[1]{\lVert#1\rVert}
\newcommand{\ind}{\mathbf 1}\newcommand{\eps}{\varepsilon}
\AtBeginDocument{\setlength{\abovedisplayskip}{6.5pt}\setlength{\belowdisplayskip}{6.5pt}\setlength{\abovedisplayshortskip}{4pt}\setlength{\belowdisplayshortskip}{4pt}}
\setlist[itemize]{leftmargin=*,itemsep=1.5pt,topsep=2pt}
\title{Gaussian Process Bandits for AI System Optimization with Selectively Verified Feedback}
\author{Shuang Wu\\UCLA\\\texttt{shuangwu222@ucla.edu}}
\date{}

\begin{document}
% Register references chronologically before the first citation.
% With unsrt this preserves the publication-year order in reference.bib.
\nocite{nitinawarat,javdani,garivierkaufmann,cheneced,baranysoberon,altschuler,fiez,leevaliant,prouhetreview,chaosnli,kirschnerpm,raychowdhurycme,bogunovic,camilleri,cbm,reda,linreboot,mukherjee,rangi,takemorikme,caikandasamy,mikkola,verma,wangpost,ares,chandak,gnnts,hbbops,aolucb,asb,cabs,fc2fb,gong,harding,impatient,laplacianbandit,mlaucb,naod,pagp,panel32,postaction,ppai,reyzin,schur,weillm,zeng}
\maketitle

\begin{abstract}
Gaussian process bandits exploit similarities among configurations to optimize prompts, retrieval systems, and agent workflows. In these applications, however, inexpensive model-based evaluations can disagree with trusted task outcomes. We study selective verification when an action is chosen before a diagnostic is observed and trusted feedback is informative on an unknown set of diagnostic states. A finite reliability class and a finite RKHS signal dictionary describe these two sources of uncertainty. We construct a fixed audit design that combines reliability coverage with reward separation and establish a finite-time regret bound that accounts for a hard audit budget and the rounds needed to obtain informative observations. A safe adaptive extension mixes this robust design with predictable Chernoff allocations while preserving a uniform identification floor and exact expected acquisition-cost accounting. For signed-coordinate kernels, structured and diffuse reliability classes require different audit scales despite sharing the same reward geometry. Synthetic experiments distinguish these certified designs from an adaptive posterior heuristic and GP controls. Safe adaptive allocation reduces aggregate regret relative to the fixed design under the exact dictionary, but its finite-budget error certificate is numerically loose. Conditional-feedback errors can remove the heuristic's advantage and target perturbations can reverse it. Separate AI-task replays assess feedback correction and fitted-dictionary policies without establishing end-to-end validation of a learned selective-verification method.
\end{abstract}

\section{Introduction}

Language-model judges make repeated evaluation inexpensive, but they can reverse the ranking of prompts, retrieval pipelines, or agent configurations. A Gaussian process (GP) fitted to such scores can estimate the wrong objective precisely. Trusted evaluation---human adjudication, execution in a protected environment, or a high-quality test suite---repairs the target but can dominate total cost. This motivates a bandit that sees cheap post-action information and selectively purchases the trusted outcome.

Existing methods cover selective audits, prediction-powered inference, auxiliary-feedback kernel bandits, and post-surrogate acquisition \cite{verma,aolucb,asb,ppai}. We ask how many trusted outcomes are required when the learner must discover \emph{where} feedback is informative. For unrestricted outcome laws, target-kernel rank and a fixed finite collection of conditional statistics do not determine that information.

We introduce a finite conditional-reliability class. A known dictionary $\{C_h\}$ contains possible informative diagnostic sets, but $h$ is unknown. Its fractional hitting number $\tau^\star$ measures the audit mass needed to touch every candidate reliable set. We couple this uncertainty to a finite RKHS signal dictionary through a max--min experimental design. The resulting rate depends on reliability coverage and on the action energy needed to distinguish reward models with different optima.

For example, RAG configurations differ in their prompts and retrieval pipelines. A cheap diagnostic can reveal request or output categories where trusted correctness distinguishes configurations. A selective policy must discover these categories and acquire informative observations while paying for every intervening deployment. This separates reward geometry, conditional reliability, and acquisition time.

Our contributions are:
\begin{itemize}
\item We prove a finite-time bound with audit scale $\tau^\star\mathfrak D_+$ and characterize the exact expected acquisition cost of a joint action--state design. A safe adaptive mixture preserves a uniform Hellinger-information floor under a pathwise audit cap. A prior-mixture converse permits fully adaptive auditing and infinite signal classes.
\item For signed rank-$d$ signals, we establish matching polynomial audit scales of $d\tau^\star/\Delta^2$ for structured reliability and $d(\tau^\star)^2/\Delta^2$ for a diffuse reliability class. The deterministic structured model admits an exact asymptotic regret coefficient expressed by a coverage problem.
\item We distinguish feedback-law error from target approximation. Two rank-one dictionaries match any prescribed finite collection of bounded feedback summaries and have identical targets, yet their regret differs by a factor growing as $K/\log K$ at a common $\log_2 K$ audit budget. Sequential audit-only controls provide reliability-free comparisons.
\item We evaluate the proved designs, an adaptive-action posterior heuristic, and GP controls across kernel geometries, overlapping reliability sets, rare states, and model perturbations. AI-task replays examine calibration and acquisition costs without establishing end-to-end validation of a learned reliability model.
\end{itemize}

\begin{figure}[t]
\centering
\includegraphics[width=\linewidth]{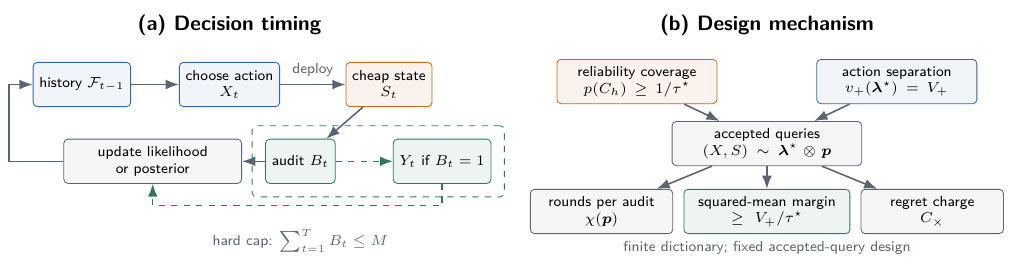}
\caption{Selective-verification protocol. (a) Given the observed history $\mathcal F_{t-1}$, actions precede diagnostics, which inform the audit decision. Dashed arrows show when a purchased outcome enters the update. (b) Reliability coverage and action separation bound the squared-mean margin per audit. The factors $\chi(\bm p)$ and $C_\times$ charge acquisition time and regret. Proposition~\ref{prop:safeadaptive} adds predictable adaptive allocation.}
\label{fig:model}
\end{figure}

\section{Related work}

\textbf{Selective verification and surrogate feedback.} Prediction-powered best-arm identification and adaptive inference combine inexpensive predictions with selectively acquired trusted labels \cite{aolucb,ppai}. Limited-verification bandits and surrogate-reward bandits study trusted feedback under attacks or machine-learning-generated rewards \cite{rangi,mlaucb}. ASB permits acquisition after observing a surrogate and accounts for learned refinements and hard budgets \cite{asb}. Confidence--budget matching and costly-observation bandits provide reward-only controls \cite{cbm,schur}. Our question concerns the additional cost of identifying an unknown reliable-state set while optimizing cumulative true reward.

\textbf{Kernel bandits and AI evaluation.} Auxiliary-feedback \texttt{IGP-UCB}, prediction-augmented \texttt{GP-UCB}, and robust multi-fidelity optimization exploit additional feedback under structural assumptions \cite{mikkola,verma,pagp}. Kernel experimental design and misspecified GP bandits characterize reward geometry and target approximation \cite{bogunovic,camilleri}. Related exploration methods include residual-bootstrap \texttt{LinReBoot}, graph-neural \texttt{GNN-TS}, and Laplacian kernelized bandits \cite{linreboot,gnnts,laplacianbandit}. ARES and HbBoPs connect trusted evaluation and kernel optimization to RAG and prompt selection \cite{ares,hbbops}. We study how conditional reliability interacts with kernel decision geometry, rather than a new kernel posterior or a generic correction estimator.

\textbf{Controlled sensing and decision-focused learning.} Controlled sensing, decision-region active learning, and transductive design allocate measurements to distinguish consequential alternatives \cite{nitinawarat,javdani,cheneced,fiez,mukherjee}. Best-arm identification with post-action context provides context-tracking and Track-and-Stop rules \cite{postaction}. These ideas motivate our likelihood decoder and adaptive allocation. Our analysis separates reliability coverage, action separation, and the deployment cost imposed by choosing an action before observing the diagnostic.

\textbf{Partial monitoring and identifiability.} Kernelized partial monitoring, conditional mean embeddings, and cost-aware streaming labels provide broader sequential abstractions \cite{kirschnerpm,raychowdhurycme,takemorikme,caikandasamy}. Exact rejection sampling implements the desired accepted-query law \cite{chandak}. Classical best-arm and mixture testing explain the polynomial rates in our structured--diffuse comparison \cite{garivierkaufmann,leevaliant,fc2fb}, while contaminated and misspecified identification establish related identifiability failures \cite{altschuler,reda}. Our results specialize these tools to a finite reliability--signal dictionary, with an adaptive converse and matched-summary counterexamples. Appendix~\ref{app:comparators} compares the assumptions and guarantees in detail.

\section{Problem formulation}\label{sec:problem}
\subsection{AI evaluation protocol and objective}
Each action is a fixed prompt, retrieval pipeline, or agent configuration used on a fresh request. A cheap diagnostic is observed after deployment, and a trusted assessment is available only if purchased. We optimize expected trusted quality under a stationary action set and kernel. Request-dependent routing and within-agent-step optimization are outside this model.

Let $\A$ and $\Sset$ be finite, with integer horizon $T\ge1$ and audit cap $M\ge1$. At round $t$, the learner chooses $X_t\in\A$, observes $S_t\in\Sset$, and then chooses $B_t\in\{0,1\}$. If $B_t=1$, it spends one audit and observes $Y_t\in\{0,1\}$. The policy satisfies $\sum_{t=1}^T B_t\le M$ almost surely. The audit decision may use $S_t$, so the lower bounds permit post-state acquisition.

\subsection{Conditional reliability and kernel signals}
Write $\Delta(\mathcal X)$ for the probability simplex on a finite set $\mathcal X$ and $[x]_+=\max\{x,0\}$. States are i.i.d. from a known action-independent law $\mu$, with $\mu_{\min}=\min_s\mu(s)>0$. A known finite family $\{C_h:h\in\mathcal H\}$ has $\mu(C_h)=\rho>0$. Its fractional hitting number is
\begin{equation}\label{eq:tau}
\tau^\star=\min_{\bm x\ge0}\left\{\sum_{s\in\Sset}x_s:\sum_{s\in C_h}x_s\ge1\ \forall h\right\}.
\end{equation}
For an optimizer, $p_s=x_s/\tau^\star$ obeys $p(C_h)\ge1/\tau^\star$.

Let $k$ be a known positive-semidefinite kernel on $\A$, with RKHS $\HH_k$. A finite signal dictionary $\mathcal U=\{u_z:z\in\mathcal Z\}\subset\HH_k$ satisfies
\[
\norm{u_z}_{\HH_k}\le B,\qquad \norm{u_z}_\infty\le U\le\tfrac14,
\qquad a_z^\star\in\arg\max_{a\in\A}u_z(a).
\]
Under environment $(h,z)$, audited outcomes are conditionally independent with
\begin{equation}\label{eq:generalmodel}
Y_t\mid(S_t=s,X_t=a)\sim\operatorname{Bern}\!\left(\tfrac12+\ind\{s\in C_h\}u_z(a)\right).
\end{equation}
The target is $f_z(a)=\tfrac12+\rho u_z(a)$ and
\[
R_T=\sum_{t=1}^T [f_z(a_z^\star)-f_z(X_t)],\qquad R_T^\star=\inf_{\text{causal policies}}\sup_{h,z}\E_{h,z}[R_T].
\]
The known centering constant is $1/2$, and the kernel models the action-dependent target component. Trusted outcomes are informative about the best action within $C_h$ and uninformative outside it. The diagnostic $S_t$ need not be a reward prediction, and the model does not assume an arbitrary judge is unbiased.

Fix one maximizing action $a_z^\star$ for each signal, including a deterministic tie rule. The learner knows the dictionaries but neither $h$ nor $z$. Equal reliable-set mass and action-independent $\mu$ make the target depend only on $z$, while $h$ controls information acquisition. This factorization need not hold for action-dependent diagnostics. Appendix~\ref{app:actionstate} extends only the acquisition rule to a known action-dependent law. Neither result learns or validates the outcome dictionary.

\subsection{Decision geometry and scope}

For $\bm{\lambda}\in\Delta(\A)$, define the action-separation margin
\begin{align}
v_+(\bm{\lambda})=\min\bigg\{&\min_z\E_{a\sim\bm{\lambda}}[u_z(a)^2],\nonumber\\[-2pt]
&\min_{a_z^\star\ne a_{z'}^\star}
\E_{a\sim\bm{\lambda}}[(u_z(a)-u_{z'}(a))^2]\bigg\},
\qquad
V_+=\max_{\bm{\lambda}} v_+(\bm{\lambda}),\quad \mathfrak D_+=V_+^{-1}.
\label{eq:dplus}
\end{align}
Writing $u_z(a)=\langle\bm\phi(a),\bm\vartheta_z\rangle$, every constraint is linear in $\bm{\lambda}$ through $\bm V(\bm{\lambda})=\sum_a\lambda_a\bm\phi(a)\bm\phi(a)^\top$. Hence $V_+$ is a pre-outcome linear program, not an oracle posterior quantity. We assume $V_+>0$ and at least two different designated optima. The features $\bm\phi(a)$ and coefficients $\bm\vartheta_z$ are vectors in a feature representation of $k$.

\section{Selective Verification with Kernel Decision Geometry}\label{sec:method}
\subsection{Reliability coverage and kernel action selection}
Two pre-outcome linear programs determine the policy. State allocation covers every candidate reliable set. Action allocation separates weak individual signals and reward pairs with different optima. Both use the known dictionaries, without knowledge of the true environment.

Let $\bm{\lambda}^\star$ optimize \eqref{eq:dplus}. For an optimizer $\bm p$ of \eqref{eq:tau}, define its action-before-state realization factor
\[
\chi(\bm p)=\max_{s\in\Sset}\frac{p_s}{\mu(s)}.
\]
Let $m_{h,z}(a,s)=\tfrac12+\ind\{s\in C_h\}u_z(a)$ and write $p_{h,z}(y\mid a,s)=m_{h,z}(a,s)^y[1-m_{h,z}(a,s)]^{1-y}$. Algorithm~\ref{alg:design} compiles the product law $\bm{\lambda}^\star\otimes\bm p$ by rejection sampling after the state arrives. The audit coin is drawn only after observing the state, independently of the unobserved outcome conditional on the action, state, and past.
For $\Lambda\ge1$, define the exact expected number of design rounds before $M$ accepts or horizon end,
\[
\mathsf N_{T,M}(\Lambda)=
\sum_{t=0}^{T-1}\Pr\{\operatorname{Bin}(t,\Lambda^{-1})<M\}.
\]

\begin{algorithm}[t]
\caption{\texttt{Reliability-Kernel-Design}}
\label{alg:design}
\begin{enumerate}[leftmargin=*,label=\arabic*.,itemsep=1pt]
\item \textbf{Input:} $T,M,\mu,\{C_h\},\{u_z\}$. Solve \eqref{eq:tau} and \eqref{eq:dplus} for $\bm p,\bm{\lambda}^\star$.
\item \textbf{Initialize:} $n\gets0$, $\ell_{h,z}\gets0$ for every $(h,z)$.
\item \textbf{For} $t=1,\ldots,T$ \textbf{do}
\item \quad\textbf{If} $n<M$ \textbf{then} draw $X_t\sim\bm{\lambda}^\star$ and observe $S_t=s$.
\item \qquad Draw $B_t\sim\operatorname{Bern}\!\left(p_s/[\chi(\bm p)\mu(s)]\right)$.
\item \qquad\textbf{If} $B_t=1$ \textbf{then} purchase $Y_t$ and set
\[
 n\gets n+1,\qquad
 \ell_{h,z}\gets\ell_{h,z}+\log p_{h,z}(Y_t\mid X_t,s)\quad\text{for all }(h,z).
\]
\item \qquad\textbf{If} $n=M$ \textbf{then} set $(\widehat h,\widehat z)\in\arg\max_{h,z}\ell_{h,z}$.
\item \quad\textbf{Else} play $X_t=a_{\widehat z}^\star$, observe $S_t$, and set $B_t=0$.
\item \quad\textbf{End if.}
\item \textbf{End for.}
\item \textbf{Return:} $a_{\widehat z}^\star$ if $n=M$ and ``insufficient audits'' otherwise.
\end{enumerate}
\end{algorithm}

\noindent\begin{minipage}{\linewidth}
\begin{proposition}[Product separation]\label{prop:separation}
For any $(h,z)$ and $(h',z')$ with $a_z^\star\ne a_{z'}^\star$,
\[
\E_{s\sim\bm p,a\sim\bm{\lambda}^\star}
[(m_{h,z}(a,s)-m_{h',z'}(a,s))^2]
\ge\frac{V_+}{\tau^\star},
\]
where $m_{h,z}$ is the Bernoulli mean in \eqref{eq:generalmodel}.
\end{proposition}
\end{minipage}\par
The proof partitions the two reliable sets into their overlap and two differences. See Appendix~\ref{proof:separation}.

The product value is also controlled from above for every fixed accepted-query design. Define
\begin{align*}
V_\Delta&=\max_{\bm{\lambda}\in\Delta(\A)}
\min_{a_z^\star\ne a_{z'}^\star}
\E_{a\sim\bm{\lambda}}[(u_z(a)-u_{z'}(a))^2],\nonumber\\[-2pt]
J^\star&=\max_{\bm{\nu}\in\Delta(\A\times\Sset)}
\min_{a_z^\star\ne a_{z'}^\star}
\E_{(a,s)\sim\bm{\nu}}[(m_{h,z}(a,s)-m_{h',z'}(a,s))^2],
\end{align*}
where the second minimum ranges over all $h,h',z,z'$ with different designated optima.

\noindent\begin{minipage}{\linewidth}
\begin{proposition}[Fixed-design sandwich]\label{prop:fixedsandwich}
The exact squared-mean fixed-design value obeys
\[
\frac{V_+}{\tau^\star}\le J^\star\le\frac{V_\Delta}{\tau^\star}.
\]
Consequently the product design is exact whenever its individual-energy constraints are nonbinding, $V_+=V_\Delta$.
\end{proposition}
\end{minipage}\par
A product dual certificate gives the upper bound (Appendix~\ref{proof:fixedsandwich}).

The individual energy constraints in \eqref{eq:dplus} are essential. Pairwise action separation alone can be hidden outside the overlap of two reliability sets. The method remains finite-dictionary and exact-kernel: $\mathcal H$, $\mathcal U$, and both designs must be frozen before outcomes are opened.
Maximum-likelihood decoding uses the margin $V_+/\tau^\star$ within controlled testing \cite{nitinawarat,fiez,mukherjee}. Adaptive sensing can outperform fixed designs in general multihypothesis problems \cite{nitinawarat}, so this comparison does not establish adaptive optimality.

\subsection{Finite-time regret with a hard audit budget}
The pathwise audit cap and expected-regret bound account for all deployed actions. Independence of accepted samples and stationary interarrival counts gives exact finite-horizon accounting. Full proofs appear in Appendix~\ref{app:proofs}.

\noindent\begin{minipage}{\linewidth}
\begin{theorem}[Finite-dictionary RKHS achievable rate]\label{thm:generalupper}
Under \eqref{eq:generalmodel} with $V_+>0$, Algorithm~\ref{alg:design} uses at most $M$ audits and satisfies
\[
\E[R_T]\le \frac{C_\times}{\chi(\bm p)}
\mathsf N_{T,M}\bigl(\chi(\bm p)\bigr)+
2\rho U\bigl[T-\mathsf N_{T,M}\bigl(\chi(\bm p)\bigr)\bigr]
|\mathcal H||\mathcal Z|
\exp\!\left(-\frac{M}{2\tau^\star\mathfrak D_+}\right).
\]
Here
\[
C_\times=\rho\chi(\bm p)\max_z
\E_{a\sim\bm{\lambda}^\star}[u_z(a_z^\star)-u_z(a)]
\le2\rho U\chi(\bm p).
\]
Since $\mathsf N_{T,M}(\Lambda)\le\min\{T,M\Lambda\}$, this refines the simpler untruncated charge whenever the horizon binds.
Thus $M\ge2\tau^\star\mathfrak D_+\log(2|\mathcal H||\mathcal Z|T)$ makes the exploitation-error term at most $\rho U$.
\end{theorem}
\end{minipage}\par
Appendix~\ref{proof:generalupper} proves the result by factoring accepted triples from their geometric waiting times and applying a Hellinger likelihood bound.

\subsection{Joint action--verification policies and acquisition costs}
The audit scale $\tau^\star\mathfrak D_+$ measures information per accepted query. In contrast, $\chi(\boldsymbol p)$ controls how quickly those queries arrive, and $C_\times/\chi(\boldsymbol p)$ is the worst expected regret per design round. Rare informative states can therefore make an audit-efficient design costly in calendar time. If $\boldsymbol p=\mu$, every round is accepted and $\chi(\boldsymbol p)=1$.

The stopped bound charges exploitation only after all $M$ observations arrive. For fixed $M$ and a long horizon, the design charge approaches $MC_\times$.

The product construction is interpretable but can be conservative. To state the exact fixed-design cost, let $\mathcal E$ be a finite environment dictionary with at least two designated optima. Under $e\in\mathcal E$, fresh outcomes at $(a,s)$ have law $P_e^{a,s}$, independently across rounds conditional on actions and states. Define $f_e(a)=\E_{S\sim\mu}[\E_e[Y\mid a,S]]$, choose $a_e^\star\in\arg\max_a f_e(a)$, and put $\Delta_e(a)=f_e(a_e^\star)-f_e(a)$. Put $G=\max_{e,a}\Delta_e(a)$. For $\bm{\nu}\in\Delta(\A\times\Sset)$ define
\begin{align*}
\Lambda(\bm{\nu})&=\sum_{a\in\A}\max_{s\in\Sset}\frac{\nu(a,s)}{\mu(s)},
& C(\bm{\nu})&=\max_e\sum_{a\in\A}\Delta_e(a)\max_s\frac{\nu(a,s)}{\mu(s)},\nonumber\\[-2pt]
I(\bm{\nu})&=\min_{e,e':a_e^\star\ne a_{e'}^\star}
\sum_{a,s}\nu(a,s)H^2(P_e^{a,s},P_{e'}^{a,s}),
\end{align*}
where $H^2(P,Q)=1-\int\sqrt{dP,dQ}$.

\noindent\begin{minipage}{\linewidth}
\begin{proposition}[Action-before-state rejection-sampling identity]\label{prop:jointdesign}
There is a stationary action and post-state audit rule whose accepted pairs are i.i.d. $\bm{\nu}$ and whose acceptance probability per round is $1/\Lambda(\bm{\nu})$. Every causal rule producing one accepted pair with law $\bm{\nu}$ has expected waiting time at least $\Lambda(\bm{\nu})$ and sampling regret at least $C_e(\bm{\nu}):=\sum_a\Delta_e(a)\max_s\nu(a,s)/\mu(s)$ under $e$; the stationary rule attains both. After the $M$th audit, choose a maximum-likelihood environment and play its designated optimum without further audits. If the horizon arrives first, remain in the design phase. Then
\[
\sup_{e\in\mathcal E}\E_e[R_T]
\le \frac{C(\bm{\nu})}{\Lambda(\bm{\nu})}\mathsf N_{T,M}\bigl(\Lambda(\bm{\nu})\bigr)
+G\bigl[T-\mathsf N_{T,M}\bigl(\Lambda(\bm{\nu})\bigr)\bigr]
(|\mathcal E|-1)e^{-MI(\bm{\nu})}.
\]
The bound therefore holds after infimizing its right-hand side over $\bm{\nu}$.
\end{proposition}
\end{minipage}\par
Appendix~\ref{proof:jointdesign} proves both the causal cost identity and the finite-horizon decoding bound.

\subsection{Adaptive verification with an information guarantee}
The fixed design can be adapted without losing worst-case identifiability. Given a fixed $\bm\nu_0$ with $I_0:=I(\bm\nu_0)>0$, fix $0<\eps\le1$ and mix any predictable design as $\bm\nu_m=\eps\bm\nu_0+(1-\eps)\widetilde{\bm\nu}_m$. Our implementation uses the Chernoff max--min law against different-optimum alternatives to the current maximum-likelihood environment.

\noindent\begin{minipage}{\linewidth}
\begin{proposition}[Safe adaptive joint design]\label{prop:safeadaptive}
Assume an exact finite observation dictionary, a known action-independent state law, and the same predictable rule under every environment. At stage $m$, hold $\bm\nu_m$ fixed until the next accepted audit and apply the row-max compiler of Proposition~\ref{prop:jointdesign}. After $M$ audits, exploit the designated optimum of a maximum-likelihood environment. Define the stage sequence by continuing the same rule beyond $T$ if necessary. Then at most $M$ audits are used pathwise and, for every $e\in\mathcal E$,
\[
\E_e[R_T]\le \E_e\!\left[\sum_{m=1}^M C_e(\bm\nu_m)\right]
+GT\min\!\left\{1,(|\mathcal E|-1)e^{-\eps M I_0}\right\}.
\]
The conditional Hellinger information is at least $\eps I_0$ for every different-optimum pair, and the process extended until $M$ audits arrive has $\E_e[\tau_M]=\E_e[\sum_{m=1}^M\Lambda(\bm\nu_m)]$.
\end{proposition}
\end{minipage}\par
Appendix~\ref{proof:safeadaptive} combines accepted-history contraction with the horizon coupling. The result composes controlled sensing \cite{nitinawarat,mukherjee} and density-ratio rejection sampling \cite{chandak} for this protocol. It requires an exact dictionary and differs from ASB's inverse-propensity variance objective \cite{asb}.

\noindent\begin{minipage}{\linewidth}
\begin{corollary}[Independent calibration transfer]\label{cor:caltransfer}
Let independent calibration data $\mathcal C$ return a finite dictionary $\widehat{\mathcal E}(\mathcal C)$ containing the exact true environment with probability at least $1-\delta$. Let $G_{\rm all}$ bound every true one-round gap, including when calibration fails ($G_{\rm all}=1$ is valid for rewards in $[0,1]$). If the instantiated policy has a nonnegative conditional bound $B_T(\mathcal C)$ on the coverage event, then $\E[R_T]\le \E[B_T(\mathcal C)]+\delta G_{\rm all}T$. Calibration with $n_{\rm cal}$ trusted labels gives total trusted-label use at most $n_{\rm cal}+M$.
\end{corollary}
\end{minipage}\par
Appendix~\ref{proof:caltransfer} proves this elementary conditioning transfer. Exact coverage is essential; continuously refit signals, misspecification, and population shift remain outside its scope.

\section{Adaptive limits and reliability sensitivity}\label{sec:limits}
\subsection{An adaptive converse and a matched coordinate family}
The converse permits actions based on the entire history and audits chosen after the current state. It therefore covers adaptive policies, rather than only the fixed designs analyzed above.

For a converse beyond finite signal dictionaries, let $\Pi_0$ contain every probability law $\pi$ on a measurable collection $\mathcal U_0\subset\HH_k$ with $\norm{u}_{\HH_k}\le B$ and $\norm{u}_\infty\le U\le1/4$ and
\begin{equation}\label{eq:zerobary}
\E_{u\sim\pi}[u(a)]=0\quad\forall a\in\A.
\end{equation}
For $\pi\in\Pi_0$, put
\[
G_\pi=\E_{u\sim\pi}[\max_a u(a)],
\qquad
W_\pi=\max_a\E_{u\sim\pi}[u(a)^2].
\]

\noindent\begin{minipage}{\linewidth}
\begin{theorem}[Prior-mixture adaptive lower bound]\label{thm:generallower}
For every $\pi\in\Pi_0$, the minimax regret over $\mathcal H\times\mathcal U_0$ satisfies
\[
R_T^\star\ge\rho T\left[G_\pi-4U\sqrt{\frac{MW_\pi}{3\tau^\star}}\right]_+.
\]
When $W_\pi>0$, $M\le3\tau^\star G_\pi^2/(64U^2W_\pi)$ implies $R_T^\star\ge\rho G_\pi T/2$.
\end{theorem}
\end{minipage}\par
The proof compares every environment with a fair-outcome null law and averages the resulting KL divergence under a zero-barycenter prior. See Appendix~\ref{proof:generallower}.

The prior may have infinite support in the RKHS. The bound respects the hard audit cap and counts every action used while waiting for an informative state. Its applicability depends on the existence of a zero-barycenter prior.

For the finite balanced signed specialization, let $\{g_j:j\in\mathcal J\}\subset\HH_k$ satisfy $\norm{g_j}_{\HH_k}\le B$ and include both signs $u_{j,\theta}=\theta g_j$. Assume $\max_a\theta g_j(a)\ge g_{\min}>0$ for every $(j,\theta)$ and $\max_j\norm{g_j}_\infty=U\le1/4$. Define
\begin{equation}\label{eq:vminus}
V_-=\max_{\bm{\lambda}\in\Delta(\A)}\min_j\E_{a\sim\bm{\lambda}}[g_j(a)^2].
\end{equation}
LP duality supplies a law on $j$ with $\max_a\E[g_j(a)^2]=V_-$. Combining it with an independent uniform sign gives a member of $\Pi_0$ with $G_\pi\ge g_{\min}$ and $W_\pi=V_-$. Therefore Theorem~\ref{thm:generallower} yields
\[
R_T^\star\ge\rho T\left[g_{\min}-4U\sqrt{\frac{MV_-}{3\tau^\star}}\right]_+.
\]
In particular, $M\le3\tau^\star g_{\min}^2/(64U^2V_-)$ implies $R_T^\star\ge\rho g_{\min}T/2$.

The coordinate family makes the two design quantities coincide. For $d\ge1$ and $0<\Delta\le\min\{B,1/4\}$, let $\A=[d]\times\{\pm1\}$,
\begin{equation}\label{eq:coordinate}
k((i,a),(j,b))=\ind\{i=j\}ab,\qquad
u_{J,\theta}(i,a)=\Delta\theta a\ind\{i=J\}.
\end{equation}
Then $V_+=V_-=\Delta^2/d$, $\mathfrak D_+=d/\Delta^2$, and $g_{\min}=U=\Delta$. For $d\ge2$, $V_\Delta=2\Delta^2/d$, so Proposition~\ref{prop:fixedsandwich} also makes the product design two-approximate for the exact squared-mean fixed-design game (the rank-one ratio is four). Appendix~\ref{proof:coordinate} derives these identities. Theorems~\ref{thm:generalupper}--\ref{thm:generallower} match at the polynomial audit scale $d\tau^\star/\Delta^2$, up to logarithmic factors.

\subsection{Structured and diffuse reliability}
The coordinate example separates reward-direction uncertainty from reliability uncertainty. Knowing the active coordinate reduces $\mathfrak D_+$ from $d/\Delta^2$ to $1/\Delta^2$ without changing the kernel. Kernel rank alone therefore does not specify the decision complexity.

The finite-dictionary assumption is consequential. Let $\Sset=[L]$ be uniform, let $1\le k\le L/2$ be an integer and $\rho=k/L$, and take as reliability class every $k$-subset $C\subseteq[L]$. This class has $\tau^\star=L/k=1/\rho$.

The structured and diffuse families share the target kernel, rank, signal family, and hitting number, but their polynomial audit scales differ by $\tau^\star$. Both satisfy \eqref{eq:generalmodel} exactly. In the diffuse family, state labels rarely repeat when $L\ge8T^2$, so their identities provide little reusable information about reliability. This regime is essential to the lower bound. The next theorem shows that selective audits cannot improve its polynomial rate.

\noindent\begin{minipage}{\linewidth}
\begin{theorem}[Diffuse-reliability frontier]\label{thm:diffuse}
Use the signed coordinate kernel and signals in \eqref{eq:coordinate}. If $M\le T$, a balanced pooled policy satisfies
\[
\E[R_T]\le2\rho\Delta M+4\rho\Delta Td
\exp\!\left\{-\frac{\lfloor M/d\rfloor\rho^2\Delta^2}{2}\right\}.
\]
Conversely, if $d\ge32$, $L\ge8T^2$, and
\[
M\le\frac{3d}{4096\rho^2\Delta^2},
\]
every adaptive action-before-state policy with post-state audits obeys
$\sup_{C,J,\theta}\E[R_T]\ge\rho\Delta T/4$.
Hence this family has audit scale $d(\tau^\star)^2/\Delta^2$ up to constants and logarithms.
\end{theorem}
\end{minipage}\par
Appendix~\ref{proof:diffuse} proves the upper bound by balanced audit-only sampling and the lower bound by an exchangeable unseen-state mixture.

\subsection{Robustness and limits of summary statistics}
Keep the independent Bernoulli protocol, with nominal and true means $m_e(a,s)$ and $\widetilde m_e(a,s)$. The true target is $\widetilde f_e(a)=\E_{S\sim\mu}[\widetilde m_e(a,S)]$, and $\widetilde R_T$ denotes regret against its maximizing action. Define
\[
D(\bm{\nu})=\min_{a_e^\star\ne a_{e'}^\star}
\left\{\sum_{a,s}\nu(a,s)[m_e(a,s)-m_{e'}(a,s)]^2\right\}^{1/2}.
\]

\noindent\begin{minipage}{\linewidth}
\begin{proposition}[Feedback-law misspecification boundary]\label{prop:feedbackmisspec}
Assume
\[
\sup_{e,a,s}|\widetilde m_e(a,s)-m_e(a,s)|\le\delta_{\rm fb},
\qquad
\sup_{e,a}|\widetilde f_e(a)-f_e(a)|\le\eps_{\rm tar}.
\]
Replace maximum likelihood in Proposition~\ref{prop:jointdesign} by the nominal squared-loss decoder. Then
\[
\sup_e\E_e[\widetilde R_T]
\le\frac{C(\bm{\nu})}{\Lambda(\bm{\nu})}\mathsf N_{T,M}\bigl(\Lambda(\bm{\nu})\bigr)
+G\bigl[T-\mathsf N_{T,M}\bigl(\Lambda(\bm{\nu})\bigr)\bigr](|\mathcal E|-1)
\exp\!\left\{-\frac{3M}{40}[D(\bm{\nu})-2\delta_{\rm fb}]_+^2\right\}
+2\eps_{\rm tar}T.
\]
Conversely, for every $0<\rho\le1/2$ and $0<\Delta\le1/4$, there is a two-action rank-one instance with $\eps_{\rm tar}=0$ and $\delta_{\rm fb}=(1-\rho)\Delta$ such that every causal policy obeys
\[
\sup_e\E_e[R_T]\ge\frac{\rho\Delta T}{2}e^{-M\kappa_{\rho\Delta}},
\qquad
\kappa_x=2x\log\frac{1+2x}{1-2x}\le16x^2\quad(0\le x\le1/4).
\]
\end{proposition}
\end{minipage}\par
Appendix~\ref{proof:feedbackmisspec} gives the squared-loss concentration argument and the target-preserving two-point construction.

Thus target error contributes an additive comparison term, while feedback-law error erodes the testing margin and can destroy identifiability. Related identifiability failures arise in contaminated best-arm identification and misspecified linear identification \cite{altschuler,reda}. For the product design, Proposition~\ref{prop:separation} gives $D(\bm{\lambda}^\star\otimes\bm p)\ge\sqrt{V_+/\tau^\star}$.

Beyond Bernoulli outcomes, identical targets and any prescribed finite feedback summaries can coexist with a regret ratio at least $K/(4\log_2 K)$ under equal $\log_2 K$ audit budgets (Proposition~\ref{prop:momentbarrier}, Appendix~\ref{app:summarybarrier}). Appendix~\ref{app:deterministic} gives an exact deterministic coverage coefficient.

\section{Experiments}\label{sec:experiments}
\subsection{Evaluation protocol and comparison methods}
We evaluate whether conditional reliability improves decisions under a fixed trusted-label budget, whether adaptive verification improves the proved fixed design, and how these benefits depend on model specification. Cumulative target regret $R_T$ charges every deployed action, including rounds without an audit. Terminal simple regret is the gap of the recommended action. Lower values are better. We separately report trusted-label counts and the round of the final audit, since equal label budgets need not imply equal acquisition time.

\textbf{Synthetic tasks.} The primary study uses signed-coordinate kernels of ranks two and eight and a periodic RBF kernel on 16 actions. It crosses one or four reliability groups, candidate-state probabilities $1$ and $.25$, and audit caps $128$, $512$, and $2{,}048$, with horizon $8{,}192$ and 256 paired replications per setting. The independent GP comparison uses 24 settings, $T=4{,}096$, $M\in\{64,256,1{,}024\}$, and 128 paired replications. A separate adaptive-design study uses eight actions, four squared-exponential kernel signals, three reliability families, $(T,M)\in\{(512,32),(2048,64)\}$, and 512 paired replications. Each primary evaluation uses prespecified dictionaries, policies, and seeds.

\textbf{Policies and controls.} \texttt{Reliability-Kernel-Design} is Algorithm~\ref{alg:design}. The safe adaptive policy in Proposition~\ref{prop:safeadaptive} mixes a fixed robust joint law with Chernoff allocation using $\eps=.25$. \texttt{Conditional TS} instead samples a reliability--signal pair from its posterior to select actions and audits all candidate-state outcomes until the cap. Its gate is fixed, so it is an adaptive-action heuristic, not the proved adaptive-verification policy. Controls include marginal-dictionary \texttt{pooled TS}, direct-arm UCB, likelihood decoders, GP policies, and a known-reliability oracle. Actions precede states, and only purchased outcomes update the models. GP controls use a Gaussian working likelihood with variance $1/4$. Appendix~\ref{app:experiments} gives the grids, default parameters, and separately development-selected parameters.

\textbf{Uncertainty and aggregation.} Policies share environment and noise streams. Per-setting paired 95\% intervals are descriptive and unadjusted. Aggregate ratios are geometric means of per-setting mean-regret ratios. Standard errors and bootstrap intervals condition on the fixed task grid, not a population of new tasks.

\subsection{Regret reduction and component comparisons}
\textbf{Conditional feedback and adaptive action selection.} In the 36-setting primary study, \texttt{conditional TS} has the lowest mean cumulative regret in 27 settings. Its paired comparison with \texttt{pooled TS} favors the conditional policy in 22 settings, is unresolved in 14, and never favors the pooled policy. The geometric regret ratios are $.556$, $.339$, and $.369$ as the audit cap increases (Figure~\ref{fig:direct}). Direct-arm UCB has higher regret in 34 settings and no resolved improvement. The overlapping-reliability study gives a ratio of $.569$, supporting the mechanism beyond disjoint reliable states.

\begin{figure}[!htbp]
\centering
\includegraphics[width=\linewidth]{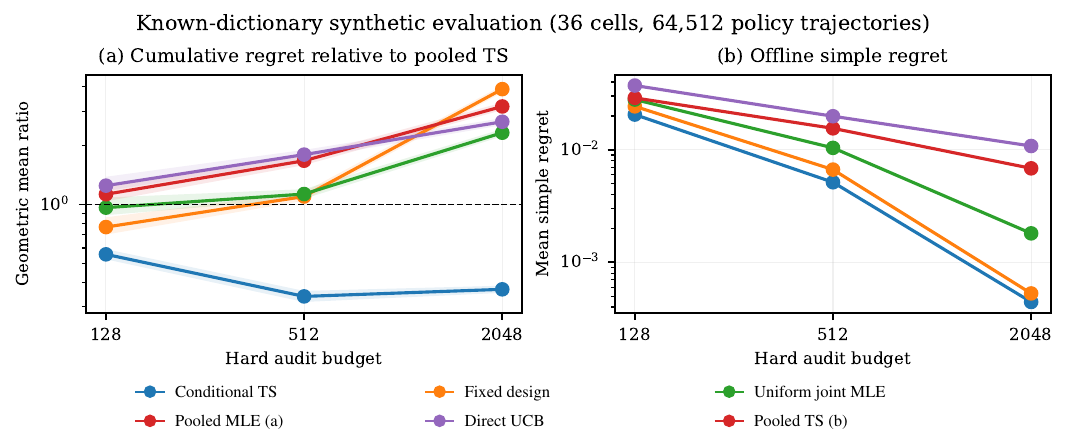}
\caption{Conditional feedback under a hard audit cap, with 256 paired replications in each of 36 settings. Left: geometric mean cumulative-regret ratios relative to \texttt{pooled TS}. Right: mean terminal simple regret. Lower values are better. Shading in (a) is a paired-bootstrap 95\% interval conditional on the fixed settings. The full setting-level comparison and audit-completion panel appear in Figure~\ref{fig:direct-full}.}
\label{fig:direct}
\end{figure}

The fixed design improves terminal identification: mean simple regret falls from $.0243$ to $.00053$ across budgets. Its mean cumulative regret is nevertheless $196.1$, versus $84.0$ for \texttt{conditional TS}$~$(at cap $2{,}048$: $261.5$ versus $24.7$). The information guarantee therefore does not imply adaptive-regret optimality. With 75\% uninformative states, only 51.3\% of either policy's trajectories reach that cap by the horizon.

\begin{table}[!htbp]
\centering\small
\caption{Independent GP comparison on 24 settings and 128 paired replications per setting. Entries average settings equally. Cumulative regret is mean $\pm$ Monte Carlo standard error, computed as $[\sum_{c=1}^{24}s_c^2/128]^{1/2}/24$, where $s_c^2$ is the within-setting sample variance. Bold marks the lowest mean in each regret column, not statistical significance. GP and identity rows use default parameters. Development-selected variants and audit counts appear in Table~\ref{tab:gpcontrols}.}
\label{tab:gpcontrols-main}
\setlength{\tabcolsep}{8pt}
\begin{tabular}{@{}lrr@{}}
\toprule
Policy & Cumulative regret $\downarrow$ & Simple regret $\downarrow$\\\midrule
\texttt{Conditional TS} & $\mathbf{92.24}\pm2.58$ & $\mathbf{.0198}$\\
\texttt{Candidate-gated pooled TS} & $99.92\pm2.58$ & $.0219$\\
\texttt{All-round joint TS} & $106.09\pm2.62$ & $.0241$\\
\texttt{All-round pooled TS} & $124.30\pm2.66$ & $.0287$\\\midrule
\texttt{GP-UCB} & $119.27\pm2.64$ & $.0275$\\
\texttt{GP-TS} & $130.40\pm2.65$ & $.0292$\\
\texttt{Identity-UCB} & $124.83\pm2.78$ & $.0286$\\
\texttt{Identity-TS} & $144.58\pm2.76$ & $.0325$\\\midrule
\texttt{Reliability-Kernel-Design} & $138.46\pm2.74$ & $.0226$\\
\bottomrule
\end{tabular}
\end{table}

\FloatBarrier
\textbf{GP controls and component effects.} Table~\ref{tab:gpcontrols-main} reports absolute regret and uncertainty for the independent comparison. \texttt{Conditional TS} has lower mean cumulative and simple regret than all listed controls. Its geometric regret ratios are $.663$ against default \texttt{GP-UCB} and $.612$ against default \texttt{GP-TS}, although UCB wins two settings. Development-selected GP parameters do not improve the aggregate result. Matched kernel-versus-identity ratios are $.925$ for UCB and $.886$ for TS, with most paired comparisons unresolved. Joint versus pooled likelihood under the same candidate-state gate has ratio $.901$. The full conditional policy therefore benefits from the combination of its model and acquisition rule, while the isolated kernel effect is smaller.

\textbf{Adaptive verification and acquisition time.} On the separately seeded exact-dictionary study, safe adaptive allocation has regret ratio $.933$ relative to the fixed robust joint design. Nine settings favor safe adaptive, 15 are unresolved, and none favor the fixed design. A subsequent comparison on new paired streams evaluates stronger controls (Table~\ref{tab:adaptive-main}). Most comparisons with audit-only TS and posterior greedy are unresolved, and pure adaptive allocation is comparable without the robust floor. The supplementary study is post-measurement evidence, not independent confirmation. Equal $T$ deployments and $M$ labels conceal longer acquisition times under selective verification. The exact-model information floor holds, but its error bound clips to one in all 24 settings.

\begin{table}[!htbp]
\centering\small
\caption{Supplementary safe-adaptive comparisons on 24 settings with 512 paired replications. Regret and time are geometric ratios to the comparator. Time is the round of the $M$th audit. Ratios below one favor safe adaptive. Every policy uses $T$ deployments and $M$ labels. Table~\ref{tab:safeadaptive} gives per-setting uncertainty classifications and decision errors.}
\label{tab:adaptive-main}
\setlength{\tabcolsep}{7pt}
\begin{tabular}{@{}lrr@{}}
\toprule
Comparator & Regret ratio $\downarrow$ & Acquisition-time ratio $\downarrow$\\\midrule
Fixed robust joint design & .916 & 2.554\\
Pure adaptive design & 1.006 & .827\\
Product design & .916 & 4.764\\
Uniform action and verification & .907 & 4.764\\
Optimized action, uniform verification & .921 & 4.764\\
Audit-only dictionary TS & .943 & 4.764\\
Audit-only posterior greedy & .966 & 4.764\\
\bottomrule
\end{tabular}
\end{table}

\subsection{Model sensitivity and AI-task evaluation}
\textbf{Reliability errors and target perturbations.} The robustness study separates changes to the conditional feedback law from changes to the reward target. Target-preserving feedback error gradually erodes the conditional policy's advantage. The conditional-to-pooled regret ratio is $1.027$ at feedback error $.20$ and $1.068$ after exact feedback flattening. Target perturbations of $.05$ and $.10$ reverse the comparison, with ratios $1.398$ and $1.799$. Correct target geometry alone therefore does not ensure informative conditional feedback, as Proposition~\ref{prop:feedbackmisspec} predicts. Appendix~\ref{app:robustness} gives the laws and uncertainty.

\textbf{AI evaluation with fitted dictionaries.} JudgeBench and the 32-judge ChaosNLI study test whether calibration can supply the model required by the online policy. The former uses 144 calibration labels and 1,584 archived judge outputs, while the latter uses 591 labels and 18,912 outputs. Online reveal budgets are additional to these calibration costs. On JudgeBench, fitted dictionaries predict worse than a pooled model and give regret ratios $11.63$ and $16.46$ against uniform verified TS. In the 32-judge study, clipping nearly every fitted signal entry makes the action design collapse to a tie, and selective verification changes acquisition time without improving regret over its matched uniform-state control.

A postdiagnostic 32-judge study separates reliability-set coverage from action-model expansion. Retaining two fitted signals while completing the reliability class improves all four settings relative to expanding both classes. The two-state reliability model has ratio $.581$ against uniform verified TS, whereas the one-state model has ratio $1.143$. This contrast identifies action-model constraints as a concrete source of the earlier loss, but it is fitted-table development evidence and does not validate a learned selective gate. Appendix~\ref{app:experiments} gives the complete comparisons, offline replays, calibration uncertainty, and data-access provenance. The evidence supports the exact-model mechanism, while practical gains remain sensitive to model estimation and acquisition time.

\clearpage
\bibliographystyle{unsrt}
\bibliography{reference}

\clearpage
\appendix

\section{Proofs}\label{app:proofs}

The proofs follow the order of the main statements. The final three proofs establish the supplementary results in Appendix~\ref{app:extensions}.

\subsection{Proof of Proposition~\ref{prop:separation}}\label{proof:separation}
Fix $e=(h,z)$ and $e'=(h',z')$ with $a_z^\star\ne a_{z'}^\star$. Let
\[
H=C_h\cap C_{h'},\qquad A=C_h\setminus C_{h'},\qquad B=C_{h'}\setminus C_h.
\]
The constant $1/2$ in the conditional means cancels. Independence of the action and state under the product design therefore gives
\begin{align*}
&\E_{s\sim\bm p,\,a\sim\bm{\lambda}^\star}
 \bigl[(m_e(a,s)-m_{e'}(a,s))^2\bigr]\\
&\quad=p(H)\E_{a\sim\bm{\lambda}^\star}\bigl[(u_z(a)-u_{z'}(a))^2\bigr]
 +p(A)\E_{a\sim\bm{\lambda}^\star}\bigl[u_z(a)^2\bigr]
 +p(B)\E_{a\sim\bm{\lambda}^\star}\bigl[u_{z'}(a)^2\bigr]\\
&\quad\ge V_+\{p(H)+p(A)+p(B)\}
 =V_+p(C_h\cup C_{h'})
 \ge V_+p(C_h)
 \ge \frac{V_+}{\tau^\star}.
\end{align*}
The first inequality uses all three classes of constraints in \eqref{eq:dplus}. The last inequality follows from the feasibility of $\bm p$ in \eqref{eq:tau}. This proves the proposition.\hfill$\square$

\subsection{Proof of Proposition~\ref{prop:fixedsandwich}}\label{proof:fixedsandwich}
Taking $\bm{\nu}=\bm{\lambda}^\star\otimes\bm p$ in the definition of $J^\star$ and applying Proposition~\ref{prop:separation} gives $J^\star\ge V_+/\tau^\star$.

For the reverse inequality, let $\mathcal P=\{(z,z'):a_z^\star\ne a_{z'}^\star\}$. Finite-game duality yields distributions $\bm q$ on $\mathcal H$ and $\bm\xi$ on $\mathcal P$ such that
\[
\sum_hq_h\ind\{s\in C_h\}\le\frac1{\tau^\star}\quad\text{for every }s,
\qquad
\sum_{(z,z')\in\mathcal P}\xi_{z,z'}[u_z(a)-u_{z'}(a)]^2
\le V_\Delta\quad\text{for every }a.
\]
Fix an arbitrary joint design $\bm{\nu}$. Restricting the alternatives to $h=h'$ and then averaging them with $\bm q\otimes\bm\xi$ gives
\begin{align*}
&\min_{h,h',\,(z,z')\in\mathcal P}
 \E_{\bm{\nu}}\bigl[(m_{h,z}(X,S)-m_{h',z'}(X,S))^2\bigr]\\
&\quad\le\sum_{s,a}\nu(a,s)
 \left(\sum_hq_h\ind\{s\in C_h\}\right)
 \left(\sum_{(z,z')\in\mathcal P}\xi_{z,z'}[u_z(a)-u_{z'}(a)]^2\right)\\
&\quad\le\frac{V_\Delta}{\tau^\star}\sum_{s,a}\nu(a,s)
 =\frac{V_\Delta}{\tau^\star}.
\end{align*}
Maximizing over $\bm{\nu}$ proves the upper bound. When $V_+=V_\Delta$, the lower and upper bounds agree.\hfill$\square$

\subsection{Proof of Theorem~\ref{thm:generalupper}}\label{proof:generalupper}
\textbf{Accepted observations and waiting times.} Fix the true environment $e=(h,z)$ and write $\chi=\chi(\bm p)$. Extend the design phase to an infinite sequence for the purpose of defining its stopping time. In each round, let $I_t$ denote the acceptance indicator generated by the rejection rule. The actual policy uses this sequence only until $M$ acceptances or the horizon. In one design round,
\[
\Pr_e\{X_t=a,S_t=s,I_t=1\}
=\lambda_a^\star\mu(s)\frac{p_s}{\chi\mu(s)}
=\frac{\lambda_a^\star p_s}{\chi}.
\]
Thus $I_t$ are independent Bernoulli$(1/\chi)$ variables and
\[
\Pr_e\{X_t=a,S_t=s\mid I_t=1\}=\lambda_a^\star p_s.
\]
Write $\tau_M=\inf\{t:\sum_{j=1}^t I_j=M\}$. The actual number of audits is at most $M$ on every sample path. Moreover,
\begin{align*}
\E_e[\min\{T,\tau_M\}]
&=\sum_{t=1}^T\Pr_e\{\tau_M\ge t\}\\
&=\sum_{t=0}^{T-1}\Pr\{\operatorname{Bin}(t,1/\chi)<M\}
=\mathsf N_{T,M}(\chi).
\end{align*}

\textbf{Regret during acquisition.} Let $\mathcal F_{t-1}$ contain all variables generated before design round $t$, and define the target gap
$\Delta_z(a)=\rho[u_z(a_z^\star)-u_z(a)]$. The event $\{\tau_M\ge t\}$ is $\mathcal F_{t-1}$-measurable, while $X_t$ has law $\bm{\lambda}^\star$ independently of that history. Hence the regret incurred before completion satisfies
\begin{align*}
\E_e\!\left[\sum_{t=1}^T\ind\{\tau_M\ge t\}\Delta_z(X_t)\right]
&=\sum_{t=1}^T\Pr_e\{\tau_M\ge t\}
 \E_{a\sim\bm{\lambda}^\star}[\Delta_z(a)]\\
&\le\frac{C_\times}{\chi}\mathsf N_{T,M}(\chi).
\end{align*}

\textbf{Likelihood comparison.} It remains to control the recommendation error. Let $Q_e$ be the law of one accepted triple $(X,S,Y)$, so that
\[
Q_e(a,s,y)=p_s\lambda_a^\star
 m_e(a,s)^y[1-m_e(a,s)]^{1-y},\qquad y\in\{0,1\}.
\]
The accepted triples are independent with law $Q_e$. For any $e'=(h',z')$ with a different designated optimum, let $L_M(e')/L_M(e)$ be the likelihood ratio of their first $M$ accepted triples. Markov's inequality and independence imply
\begin{align*}
\Pr_e\{L_M(e')\ge L_M(e)\}
&\le\E_e\!\left[\sqrt{L_M(e')/L_M(e)}\right]\\
&=\left(\sum_{s,a,y}\sqrt{Q_e(a,s,y)Q_{e'}(a,s,y)}\right)^M\\
&=\left(1-\E_{s\sim\bm p,\,a\sim\bm{\lambda}^\star}
 [H^2(P_e^{a,s},P_{e'}^{a,s})]\right)^M\\
&\le\exp\!\left\{-\frac M2
 \E_{s\sim\bm p,\,a\sim\bm{\lambda}^\star}[(m_e(a,s)-m_{e'}(a,s))^2]\right\}\\
&\le\exp\!\left(-\frac{MV_+}{2\tau^\star}\right).
\end{align*}
Here $H^2(P,Q)=1-\sum_y\sqrt{P(y)Q(y)}$. The penultimate line uses $1-x\le e^{-x}$ and the Bernoulli inequality $H^2(\operatorname{Bern}(x),\operatorname{Bern}(y))\ge(x-y)^2/2$. The final line is Proposition~\ref{prop:separation}.

Let $E$ be the event that the maximum-likelihood recommendation has a different designated optimum. On $E$, some such $e'$ has likelihood at least that of $e$, including ties. A union bound gives
\[
\Pr_e(E)\le|\mathcal H||\mathcal Z|
 \exp\!\left(-\frac{MV_+}{2\tau^\star}\right).
\]
\textbf{Stopping and exploitation.} To combine this error probability with the remaining horizon, let $G_j$ be the waiting time for the $j$th accepted triple $(\bar X_j,\bar S_j,\bar Y_j)$. The repeated independent rejection rule gives, for every integer $g\ge1$ and measurable set $D$ of triples,
\[
\Pr_e\{G_j=g,\ (\bar X_j,\bar S_j,\bar Y_j)\in D\}
=(1-1/\chi)^{g-1}\chi^{-1}Q_e(D).
\]
The pairs indexed by $j$ are independent. Thus $\tau_M=\sum_{j=1}^M G_j$ is independent of the accepted sample and of $E$. Since each target gap is at most $2\rho U$, exploitation regret is at most
\begin{align*}
2\rho U\E_e[(T-\tau_M)_+\ind\{E\}]
&=2\rho U\{T-\mathsf N_{T,M}(\chi)\}\Pr_e(E)\\
&\le2\rho U\{T-\mathsf N_{T,M}(\chi)\}
 |\mathcal H||\mathcal Z|e^{-M/(2\tau^\star\mathfrak D_+)}.
\end{align*}
Adding the design and exploitation terms proves the regret bound. Finally, $\E_e[\tau_M]=M\chi$ gives $\mathsf N_{T,M}(\chi)\le\min\{T,M\chi\}$. Substituting the stated audit threshold makes the last exponential at most $1/(2|\mathcal H||\mathcal Z|T)$, so the exploitation term is at most $\rho U$.\hfill$\square$

\subsection{Proof of Proposition~\ref{prop:jointdesign}}\label{proof:jointdesign}
Write
\[
m_a=\max_{s\in\Sset}\frac{\nu(a,s)}{\mu(s)},\qquad
\Lambda=\sum_{a\in\A}m_a.
\]
Actions with $m_a=0$ are never selected. For $m_a>0$, define
\[
q_a=\frac{m_a}{\Lambda},
\qquad
\alpha_{a,s}=\frac{\nu(a,s)}{m_a\mu(s)}.
\]
Because $m_a\ge\nu(a,s)/\mu(s)$, we have $0\le\alpha_{a,s}\le1$. A design round draws $X_t\sim\bm q$, observes $S_t$, and audits with probability $\alpha_{X_t,S_t}$. Hence
\begin{align*}
\Pr\{X_t=a,S_t=s,B_t=1\}
&=q_a\mu(s)\alpha_{a,s}
=\frac{\nu(a,s)}{\Lambda},\\
\Pr\{B_t=1\}&=\frac1\Lambda,
\qquad
\Pr\{X_t=a,S_t=s\mid B_t=1\}=\nu(a,s).
\end{align*}
Independent repetition therefore produces i.i.d. accepted pairs with law $\bm{\nu}$.

\textbf{Causal lower bound.} We next prove the causal lower bound. Consider any policy that produces one audited pair with law $\bm{\nu}$. Let $\mathcal F_{t-1}$ contain the history before action $X_t$ is chosen, let $\tau$ be the first audited round, and define
\[
N_a=\sum_{t\ge1}\ind\{\tau\ge t,X_t=a\}.
\]
The event $\{\tau\ge t,X_t=a\}$ is measurable with respect to $\sigma(\mathcal F_{t-1},X_t)$. Since $S_t\sim\mu$ independently of $\mathcal F_{t-1}$ and $X_t$, Tonelli's theorem gives, for every $(a,s)$,
\begin{align*}
\E\!\left[\sum_{t\ge1}\ind\{\tau\ge t,X_t=a,S_t=s\}\right]
&=\sum_{t\ge1}\E\!\left[\ind\{\tau\ge t,X_t=a\}
 \Pr\{S_t=s\mid\mathcal F_{t-1},X_t\}\right]\\
&=\mu(s)\E[N_a].
\end{align*}
The audited occurrence $(X_\tau,S_\tau)$ is one term in this sum and has law $\bm{\nu}$. Therefore
\[
\nu(a,s)\le\mu(s)\E[N_a],
\qquad
\E[N_a]\ge\max_s\frac{\nu(a,s)}{\mu(s)}=m_a.
\]
Summing over $a$ yields
\[
\E[\tau]=\sum_a\E[N_a]\ge\sum_am_a=\Lambda.
\]
Under environment $e$, weighting the same inequality by $\Delta_e(a)$ yields the sampling-regret bound
\[
\E_e\!\left[\sum_{t=1}^{\tau}\Delta_e(X_t)\right]
=\sum_a\Delta_e(a)\E_e[N_a]
\ge\sum_a\Delta_e(a)m_a=C_e(\bm{\nu}).
\]
For the stationary construction, $\E[\tau]=\Lambda$ and $\E[N_a]=q_a\Lambda=m_a$, so both inequalities are equalities. The argument applies conditionally at each stage when the next design law is measurable with respect to the accepted history.

\textbf{Finite-horizon regret.} It remains to prove the finite-horizon regret bound. Before the $M$th audit, one design round has expected regret
\[
\sum_aq_a\Delta_e(a)=\frac{C_e(\bm{\nu})}{\Lambda}\le\frac{C(\bm{\nu})}{\Lambda}.
\]
The accept indicators are i.i.d. Bernoulli$(1/\Lambda)$. Repeating the stopping-time calculation in Appendix~\ref{proof:generalupper}, the expected design length is $\mathsf N_{T,M}(\Lambda)$ and the expected exploitation length is $T-\mathsf N_{T,M}(\Lambda)$.

Let $Q_e$ denote the law of one accepted triple. For a wrong-decision environment $e'$, its Bhattacharyya affinity with $Q_e$ is
\begin{align*}
\sum_{a,s}\nu(a,s)\int\sqrt{dP_e^{a,s}dP_{e'}^{a,s}}
&=1-\sum_{a,s}\nu(a,s)H^2(P_e^{a,s},P_{e'}^{a,s})\\
&\le1-I(\bm{\nu})\le e^{-I(\bm{\nu})}.
\end{align*}
For the first $M$ accepted triples, Markov's inequality applied to the square root of the likelihood ratio and independence give
\[
\Pr_e\{L_M(e')\ge L_M(e)\}
\le e^{-MI(\bm{\nu})}.
\]
A maximum-likelihood decision with the wrong optimum implies this event for at least one of the at most $|\mathcal E|-1$ alternatives. A union bound therefore gives error probability at most $(|\mathcal E|-1)e^{-MI(\bm{\nu})}$. The accepted triples are independent of their geometric waiting times, so this event is independent of the exploitation length. Multiplying its probability by the expected exploitation length and the largest one-round gap $G$, then adding the design regret, proves the displayed bound. If fewer than $M$ audits arrive by time $T$, the policy remains in the design phase and has no exploitation segment.\hfill$\square$

\subsection{Proof of Proposition~\ref{prop:safeadaptive}}\label{proof:safeadaptive}
Fix the true environment $e$ and a different-optimum alternative $e'$. Let $\mathcal F^{\rm aud}_{m-1}$ be the observed history before stage $m$, including internal randomization, and let $L_m(e')/L_m(e)$ be the likelihood ratio after $m$ accepted observations. The law $\bm\nu_m$ is measurable with respect to $\mathcal F^{\rm aud}_{m-1}$ and is selected by the same rule under both environments. Therefore the design factors cancel from the likelihood ratio. Conditional on $\mathcal F^{\rm aud}_{m-1}$, we have
\begin{align*}
&\E_e\!\left[\left.\sqrt{\frac{L_m(e')}{L_m(e)}}\,\right|\mathcal F^{\rm aud}_{m-1}\right]\\
&\quad=\sqrt{\frac{L_{m-1}(e')}{L_{m-1}(e)}}
\sum_{a,s}\nu_m(a,s)\int\sqrt{dP_e^{a,s}dP_{e'}^{a,s}}\\
&\quad=\sqrt{\frac{L_{m-1}(e')}{L_{m-1}(e)}}
\left[1-\sum_{a,s}\nu_m(a,s)H^2(P_e^{a,s},P_{e'}^{a,s})\right].
\end{align*}
By nonnegativity and the mixture definition,
\[
\sum_{a,s}\nu_m(a,s)H^2(P_e^{a,s},P_{e'}^{a,s})
\ge\eps\sum_{a,s}\nu_0(a,s)H^2(P_e^{a,s},P_{e'}^{a,s})
\ge\eps I_0.
\]
Iterating conditional expectations, using $1-x\le e^{-x}$, and applying Markov's inequality gives
\[
\Pr_e\{L_M(e')\ge L_M(e)\}
\le \E_e\!\left[\sqrt{\frac{L_M(e')}{L_M(e)}}\right]
\le (1-\eps I_0)^M\le e^{-\eps M I_0}.
\]
A maximum-likelihood recommendation with the wrong optimum requires this event for at least one of at most $|\mathcal E|-1$ alternatives. Hence its probability is at most
\[
\min\!\left\{1,(|\mathcal E|-1)e^{-\eps M I_0}\right\}.
\]

\textbf{Acquisition cost and the finite horizon.} Extend the design process beyond $T$ until it collects all $M$ audits. Conditional on $\mathcal F^{\rm aud}_{m-1}$, Proposition~\ref{prop:jointdesign} gives expected waiting time $\Lambda(\bm\nu_m)$ and expected design-stage regret $C_e(\bm\nu_m)$. The tower property and Tonelli's theorem give
\[
\E_e[\tau_M]=\E_e\!\left[\sum_{m=1}^M\Lambda(\bm\nu_m)\right],\qquad
\E_e[R_{\rm design}^{\rm ext}]=\E_e\!\left[\sum_{m=1}^M C_e(\bm\nu_m)\right].
\]
The horizon-$T$ design regret is at most $R_{\rm design}^{\rm ext}$ pathwise. If all audits arrive by $T$, erroneous exploitation implies an error of the extended $M$-audit decoder and costs at most $GT$. If they do not arrive, there is no exploitation segment. Adding the two terms proves the regret bound. This coupling does not assume independence between completion time and decoding error. The algorithm checks the remaining budget before every purchase, so the cap holds pathwise.\hfill$\square$

\subsection{Proof of Corollary~\ref{cor:caltransfer}}\label{proof:caltransfer}
Let $E=\{e_\star\in\widehat{\mathcal E}(\mathcal C)\}$ be the exact-coverage event. Because $E$ is measurable with respect to the independent calibration sample, the conditional online guarantee and the pathwise bound $0\le R_T\le G_{\rm all}T$ give
\begin{align*}
\E[R_T]
&=\E[R_T\ind\{E\}]+\E[R_T\ind\{E^c\}]\\
&\le \E[B_T(\mathcal C)\ind\{E\}]+G_{\rm all}T\Pr(E^c)\\
&\le \E[B_T(\mathcal C)]+\delta G_{\rm all}T.
\end{align*}
The calibration phase consumes $n_{\rm cal}$ trusted labels before the online audit cap $M$, so the total trusted-label endowment is $n_{\rm cal}+M$.\hfill$\square$

\subsection{Proof of Theorem~\ref{thm:generallower}}\label{proof:generallower}
Fix an arbitrary adaptive policy with at most $M$ audits on every sample path. The dual of \eqref{eq:tau} supplies a distribution $\bm q$ over $h$ such that
\begin{equation}\label{eq:reliabilitydual}
\Pr_{h\sim\bm q}\{s\in C_h\}\le1/\tau^\star\qquad\forall s.
\end{equation}
Fix $\pi\in\Pi_0$ and put the product prior $\bm q\otimes\pi$ on $E=(h,u)$. For each $e=(h,u)$, let $P_e$ be the policy transcript law under the conditional mean in \eqref{eq:generalmodel}. Let $P_0$ be the transcript law obtained by replacing every audited outcome with an independent $\operatorname{Bern}(1/2)$ variable while leaving the state process and the policy's internal randomization unchanged.

Let $\mathcal F_{t-1}$ contain the transcript before $X_t$ is chosen, and let $B_t$ be the audit indicator selected after observing $S_t$. Under $P_0$, the latent variable $E$ is independent of $\mathcal F_{t-1}$, $X_t$, $S_t$, and $B_t$. By \eqref{eq:reliabilitydual}, the zero-barycenter condition, and the definition of $W_\pi$, every audited query satisfies
\[
\E_{E}\!\left[\ind\{S_t\in C_h\}u(X_t)^2
\mid\mathcal F_{t-1},X_t,S_t,B_t=1\right]
\le\frac{W_\pi}{\tau^\star}.
\]
For $|x|\le1/4$,
\begin{align*}
\operatorname{kl}\!\left(\operatorname{Bern}(1/2),
 \operatorname{Bern}(1/2+x)\right)
&=-\frac12\log(1-4x^2)\\
&\le\frac{8}{3}x^2,
\end{align*}
where the inequality holds because $4x^2\le1/4$ and $-\log(1-y)\le4y/3$ for $0\le y\le1/4$. The adaptive KL chain rule, followed by the preceding conditional bound and the pathwise constraint $\sum_{t=1}^TB_t\le M$, yields
\begin{align*}
\E_E[\operatorname{KL}(P_0\Vert P_E)]
&=\E_{0,E}\!\left[
 \sum_{t=1}^TB_t\operatorname{kl}\!\left(
 \operatorname{Bern}(1/2),
 \operatorname{Bern}\{1/2+\ind(S_t\in C_h)u(X_t)\}
 \right)\right]\\
&\le\frac83\E_0\!\left[
 \sum_{t=1}^TB_t\frac{W_\pi}{\tau^\star}\right]
\le\frac{8MW_\pi}{3\tau^\star}.
\end{align*}

For a fixed environment $e=(h,u)$ and transcript, define the target regret functional
\[
L_e=\rho\sum_{t=1}^T\{\max_a u(a)-u(X_t)\}.
\]
It satisfies $0\le L_e\le2\rho UT$. Under $P_0$, the policy actions are independent of $u$, so \eqref{eq:zerobary} gives
\begin{align*}
\E_E[\E_0[L_E]]
&=\rho\sum_{t=1}^T
 \left\{\E_{u\sim\pi}[\max_a u(a)]
 -\E_0[\E_{u\sim\pi}[u(X_t)]]\right\}\\
&=\rho G_\pi T.
\end{align*}
For each $e$, the variational characterization of total variation gives
\[
\E_e[L_e]\ge\E_0[L_e]-2\rho UT\operatorname{TV}(P_0,P_e).
\]
Averaging over $E$, using Pinsker's inequality, Jensen's inequality, and the KL bound above, we obtain
\begin{align*}
\E_{e\sim\bm q\otimes\pi}[\E_{P_e}[L_e]]
&\ge\rho G_\pi T-2\rho UT\E_E[\operatorname{TV}(P_0,P_E)]\\
&\ge\rho G_\pi T-2\rho UT
 \sqrt{\frac12\E_E[\operatorname{KL}(P_0\Vert P_E)]}\\
&\ge\rho T\left(G_\pi-4U\sqrt{\frac{MW_\pi}{3\tau^\star}}\right).
\end{align*}
The left-hand side is the Bayes regret of the fixed policy, and Bayes regret lower-bounds its worst-environment regret. Since regret is nonnegative, taking the positive part and then infimizing over policies proves the first claim. If $W_\pi=0$, every action has $u(a)=0$ almost surely under $\pi$, so $G_\pi=0$ and the bound is trivial. Otherwise, if $M\le3\tau^\star G_\pi^2/(64U^2W_\pi)$, the square-root term is at most $G_\pi/2$, which proves the second claim.\hfill$\square$

\subsection{Coordinate specialization of the design bounds}\label{proof:coordinate}
For a design $\bm{\lambda}$ on $[d]\times\{\pm1\}$, write
\[
r_i=\bm{\lambda}(i,+1)+\bm{\lambda}(i,-1),
\qquad \sum_{i=1}^dr_i=1.
\]
The sign assigned within a coordinate does not affect any squared signal. For $u_{J,\theta}$,
\[
\E_{(i,a)\sim\bm{\lambda}}[u_{J,\theta}(i,a)^2]=\Delta^2r_J.
\]
Hence the energy part of $v_+(\bm{\lambda})$ is $\Delta^2\min_Jr_J\le\Delta^2/d$. Under the uniform coordinate design $r_i=1/d$. If two signals use distinct coordinates, their squared contrast is $\Delta^2(r_J+r_{J'})=2\Delta^2/d$. If they use the same coordinate with opposite signs, it is $4\Delta^2r_J=4\Delta^2/d$. All pairs with different designated optima fall into one of these cases. The energy constraint is therefore the active one and $V_+=\Delta^2/d$.

The functions $g_J(i,a)=\Delta a\ind\{i=J\}$ satisfy
$\E_{\bm{\lambda}}[g_J^2]=\Delta^2r_J$, so the identical max--min calculation in \eqref{eq:vminus} gives $V_-=\Delta^2/d$. Finally, if $d\ge2$, the pure pairwise value $V_\Delta$ is at most $2\Delta^2/d$ because the two smallest coordinate masses sum to at most $2/d$. Uniform coordinate mass attains this bound, as the two contrast cases above show. Thus $V_\Delta=2\Delta^2/d$. When $d=1$, only the opposite-sign case remains and $V_\Delta=4\Delta^2$.\hfill$\square$

\subsection{Proof of Theorem~\ref{thm:diffuse}}\label{proof:diffuse}

For the all-$k$-subset class, uniform weights $x_s=1/k$ are feasible in \eqref{eq:tau}. Conversely, averaging any feasible constraint over a uniformly random $k$-subset gives $(k/L)\sum_sx_s\ge1$. Hence $\tau^\star=L/k=1/\rho$.

For $M<d$, the stated upper bound exceeds the maximal regret $2\rho\Delta T$, so any recommendation suffices. For $M\ge d$, audit the first $M$ rounds and assign at least $n=\lfloor M/d\rfloor$ sign-$+1$ queries to every coordinate. Marginalizing the i.i.d. state gives, independently across queries,
\[
Y\mid I=i,A=+1\sim
\operatorname{Bern}\!\left(\tfrac12+\rho\theta\Delta\ind\{i=J\}\right).
\]
Choose the coordinate with largest absolute centered empirical mean and choose its sign from that mean. The active coordinate has centered expectation $\theta\rho\Delta$, while all other coordinates have expectation zero. If every empirical mean differs from its expectation by strictly less than $\rho\Delta/2$, the active coordinate and sign are recovered. For $n$ Bernoulli samples, Hoeffding's inequality and a union bound give
\[
\Pr\{\widehat J\ne J\text{ or }\widehat\theta\ne\theta\}
\le2d\exp\!\left(-\frac{n\rho^2\Delta^2}{2}\right).
\]
Each exploration round and each round after an incorrect recommendation incurs at most $2\rho\Delta$. The first $M$ rounds therefore contribute at most $2\rho\Delta M$, and the remaining rounds contribute at most the second term in the theorem.

\textbf{Diffuse-reliability lower bound.} For the converse, let $P_0$ be the transcript law with every audited outcome fair. For each coordinate $j$, let $Q_j$ mix the sign-$+1$ environments uniformly over all $k$-subsets. Define $Q_j^\dagger$ by using the same state sequence and policy randomization as $Q_j$, but returning a fair outcome whenever a state label repeats. Before a repeat, after $t-1$ distinct labels have appeared, a newly observed label has posterior active probability at most
\[
\frac{k}{L-t+1}\le2\rho.
\]
This remains true after conditioning on adaptive actions and post-state audit decisions because the unobserved labels are exchangeable and at most $k$ active labels remain. Conditional on a first-seen state, an audit of coordinate $j$ under $Q_j^\dagger$ has mean $1/2+\sigma p_t\Delta$, where $\sigma\in\{\pm1\}$ is the played sign and $0\le p_t\le2\rho$. Thus its absolute displacement from $1/2$ is at most $2\rho\Delta\le1/4$. Audits of other coordinates are fair. For $|x|\le1/4$,
\[
\operatorname{kl}\{\operatorname{Bern}(1/2),\operatorname{Bern}(1/2+x)\}
\le\frac83x^2.
\]
Let $N_j$ be the number of audits on coordinate $j$. The adaptive KL chain rule, with all conditionals taken under $P_0$, and the pathwise budget yield
\[
\sum_{j=1}^d\operatorname{KL}(P_0\Vert Q_j^\dagger)
\le\frac{32}{3}\rho^2\Delta^2\E_0[\sum_jN_j]
\le\frac{32}{3}\rho^2\Delta^2M.
\]
The coordinate play counts $T_j$ satisfy $\sum_jT_j=T$. If $j$ is uniform, Markov's inequality shows that each of the events $\{\E_0[T_j]>4T/d\}$ and $\{\operatorname{KL}(P_0\Vert Q_j^\dagger)>128\rho^2\Delta^2M/(3d)\}$ contains fewer than one quarter of the coordinates. Hence at least one $j$ satisfies both
\[
\E_0[T_j]\le\frac{4T}{d},\qquad
\operatorname{KL}(P_0\Vert Q_j^\dagger)
\le\frac{128\rho^2\Delta^2M}{3d}\le\frac1{32}.
\]
For $E_j=\{T_j\le T/2\}$, Markov's inequality gives $P_0(E_j^c)\le8/d\le1/4$. The budget condition makes the displayed KL at most $1/32$, so Pinsker's inequality gives
\[
Q_j^\dagger(E_j)\ge P_0(E_j)-\sqrt{\operatorname{KL}(P_0\Vert Q_j^\dagger)/2}
\ge\frac34-\frac18=\frac58.
\]
The state sequence is common under $Q_j$ and $Q_j^\dagger$. Their transcripts can differ only after a repeated state, whose probability is at most $T(T-1)/(2L)\le1/16$. Hence $Q_j(E_j)\ge9/16\ge1/2$. On $E_j$, at least $T/2$ actions use another coordinate, and every such action incurs target regret $\rho\Delta$. The Bayes regret under $Q_j$ is therefore at least $\rho\Delta(T/2)(1/2)=\rho\Delta T/4$. Since $Q_j$ is a uniform mixture over fixed $k$-subsets, at least one fixed $C$ attains this lower bound. The sign-$+1$ subfamily lies inside the stated signed class.\hfill$\square$

\subsection{Proof of Proposition~\ref{prop:feedbackmisspec}}\label{proof:feedbackmisspec}
Fix the true environment $e$ and a wrong-decision $e'$. For one accepted query $x=(a,s)\sim\bm{\nu}$, put $d(x)=m_e(x)-m_{e'}(x)$ and
\[
W=\{Y-m_{e'}(x)\}^2-\{Y-m_e(x)\}^2
=2d(x)\left\{Y-\frac{m_e(x)+m_{e'}(x)}2\right\}.
\]
Writing $D_{e,e'}^2=\E_{\bm{\nu}}[d(x)^2]$ and $r_e(x)=\widetilde m_e(x)-m_e(x)$, Cauchy--Schwarz gives
\[
\overline W:=\E_e[W]
=D_{e,e'}^2+2\E_{\bm{\nu}}[d(x)r_e(x)]
\ge D_{e,e'}(D_{e,e'}-2\delta_{\rm fb}).
\]
When $D_{e,e'}\le2\delta_{\rm fb}$ the displayed upper bound is trivial. Otherwise, $|W|\le2|d(x)|$, $\E[W^2]\le4D_{e,e'}^2$, $\overline W\le D_{e,e'}^2+2\delta_{\rm fb}D_{e,e'}<2D_{e,e'}^2$, and $|W-\overline W|\le4$. Bernstein's inequality yields
\[
\Pr_e\!\left\{\sum_{m=1}^MW_m\le0\right\}
\le\exp\!\left\{-\frac{M\overline W^2}
{2(4D_{e,e'}^2+4\overline W/3)}\right\}
\le\exp\!\left\{-\frac{3M}{40}
(D_{e,e'}-2\delta_{\rm fb})^2\right\}.
\]
A union bound over wrong-decision environments, followed by the same schedule calculation as in Proposition~\ref{prop:jointdesign}, proves the nominal-regret term. For every action, replacing $f_e$ by $\widetilde f_e$ changes its gap by at most $2\eps_{\rm tar}$, which adds at most $2\eps_{\rm tar}T$.

\textbf{Target-preserving converse.} For the converse, take actions $a\in\{-1,+1\}$, signs $\theta\in\{-1,+1\}$, and a known set $C$ with $\mu(C)=\rho$. The nominal and true conditional means are
\[
m_\theta(a,s)=\tfrac12+\ind\{s\in C\}\theta a\Delta,
\qquad
\widetilde m_\theta(a,s)=\tfrac12+\rho\theta a\Delta.
\]
Both have target $f_\theta(a)=\tfrac12+\rho\theta a\Delta$, so $\eps_{\rm tar}=0$, while their uniform conditional discrepancy is $(1-\rho)\Delta$. The true feedback is state independent. Let $P_+$ and $P_-$ be the complete transcript laws under the two signs. Conditional on any audited action, the outcome laws are $\operatorname{Bern}(1/2+\rho\Delta)$ and $\operatorname{Bern}(1/2-\rho\Delta)$, possibly in the opposite order. Unaudited rounds have identical conditional laws. The adaptive KL chain rule and the pathwise constraint therefore give
\[
\operatorname{KL}(P_+\Vert P_-)
\le \E_+\!\left[\sum_{t=1}^TB_t\right]\kappa_{\rho\Delta}
\le M\kappa_{\rho\Delta}.
\]
Apply Bretagnolle--Huber to the event $\{X_t=+1\}$ using the transcript available before $X_t$ is chosen. Data processing bounds its KL by the full-transcript KL, so
\[
P_+\{X_t=-1\}+P_-\{X_t=+1\}
\ge\frac12e^{-M\kappa_{\rho\Delta}}.
\]
The two one-round regret gaps are $2\rho\Delta$. Summing over $t$ and averaging the two signs yields
\[
\frac{\E_+[R_T]+\E_-[R_T]}{2}
=\rho\Delta\sum_{t=1}^T
\bigl[P_+\{X_t=-1\}+P_-\{X_t=+1\}\bigr]
\ge\frac{\rho\Delta T}{2}e^{-M\kappa_{\rho\Delta}}.
\]
Finally, $\kappa_x=4x\operatorname{artanh}(2x)\le16x^2$ for $x\le1/4$. The maximum regret is at least the two-point average, which proves the converse.\hfill$\square$

\subsection{Proof of Proposition~\ref{prop:momentbarrier}}\label{proof:momentbarrier}
Index environments by leaves $h\in\{0,1\}^d$ and put $\theta_h=(-1)^{\sum_i h_i}$. Besides decision actions $\{+1,-1\}$, include one diagnostic $q_v$ for every internal binary-tree node. Decision outcomes have law $\operatorname{Bern}(1/2+a\theta_h\Delta)$.

\textbf{Matching the feedback summaries.} Fix distinct $x_1,\ldots,x_{4p+5}\in(0,1/2)$ and map each point into $\R^{2p+1}$ by
\[
\bm F(x)=\bigl(x,\psi_1(x),\ldots,\psi_p(x),
\psi_1(1-x),\ldots,\psi_p(1-x)\bigr).
\]
Tverberg's theorem partitions these $4p+5=2\{(2p+1)+1\}+1$ points into three nonempty sets $I_0,I_1,I_2$ whose convex hulls intersect \cite{baranysoberon}. Hence there are probability weights $w_{j,i}$ supported on $I_j$ with $\sum_{i\in I_j}w_{j,i}\bm F(x_i)$ independent of $j$. Define
\[
Q_j=\frac12\sum_{i\in I_j}w_{j,i}
  \bigl(\delta_{x_i}+\delta_{1-x_i}\bigr).
\]
The three laws have pairwise disjoint supports, common mean $1/2$, and identical expectations of every $\psi_r$: the latter is half the matched $\psi_r(x)$ coordinate plus half the matched $\psi_r(1-x)$ coordinate. Their combined support has at most $8p+10$ points. Choosing $\psi_r(y)=y^r$ recovers finite raw-moment matching; the equal-weight Prouhet construction gives a separate explicit instance \cite{prouhetreview}.

\textbf{Equal targets and different information.} In the coded dictionary, $q_v$ returns $Q_{h_{|v|+1}}$ when $v$ is a prefix of $h$, and $Q_2$ otherwise. In the uninformative dictionary every diagnostic returns $Q_0$. Their common target is $f_h(a)=1/2+a\theta_h\Delta$ on decisions and $f_h(q_v)=1/2$ on diagnostics. With $\phi(+1)=1$, $\phi(-1)=-1$, and $\phi(q_v)=0$, the centered target is $\theta_h\Delta\phi$ under the rank-one kernel $k(x,x')=\phi(x)\phi(x')$.

For the uninformative dictionary, fix opposite-parity leaves and let $P_+,P_-$ be their transcript laws. Only audited decision actions distinguish them. The adaptive KL chain rule and the pathwise budget give
\[
\operatorname{KL}(P_+\|P_-)\le M\kappa_\Delta,\qquad
\kappa_\Delta=\operatorname{kl}\!\left\{\operatorname{Bern}(1/2+\Delta),\operatorname{Bern}(1/2-\Delta)\right\}.
\]
\textbf{Achievable regret in the informative dictionary.} The coded policy starts at the root. At depth $r$, it audits $q_v$ for the prefix $v$ already recovered. The resulting sample belongs to the disjoint support of $Q_0$, $Q_1$, or $Q_2$. Because the queried node is on the true path, only $Q_0$ or $Q_1$ can occur, and its support identifies bit $h_{r+1}$ from one sample. After $d$ audits the leaf and its parity are known. Each diagnostic action has target value $1/2$ while the optimal decision has value $1/2+\Delta$, so these $d$ rounds incur exactly $d\Delta$ regret. The policy then plays the optimal sign.

\textbf{Lower bound in the uninformative dictionary.} For the uninformative dictionary, map decision actions to their signs and map every diagnostic action to an independent fair sign. Under either environment, the expected instantaneous regret is exactly $2\Delta$ times the probability that this mapped sign is wrong. For each $t$, apply Bretagnolle--Huber to the mapped sign based on the pre-action transcript. Data processing and the preceding KL bound give
\[
P_+\{\widehat\theta_t=-1\}+P_-\{\widehat\theta_t=+1\}
\ge\frac12e^{-M\kappa_\Delta}.
\]
Consequently,
\[
\frac{\E_+[R_T]+\E_-[R_T]}{2}
=\Delta\sum_{t=1}^T
\bigl[P_+\{\widehat\theta_t=-1\}+P_-\{\widehat\theta_t=+1\}\bigr]
\ge\frac{\Delta T}{2}e^{-M\kappa_\Delta}.
\]
Thus one of the two leaves attains the stated lower bound. Finally,
$\kappa_\Delta=4\Delta\operatorname{artanh}(2\Delta)\le16\Delta^2$ for $\Delta\le1/4$. Under $M=d$, $T=K$, and $\Delta=1/(4\sqrt K)$, we have $M\kappa_\Delta\le d/K\le1/2$. Since $e^{-1/2}/2\ge1/4$, the uninformative regret is at least $\Delta K/4$, whereas the coded regret is $d\Delta$. Their ratio is at least $K/(4d)=K/(4\log_2K)$.\hfill$\square$

\subsection{Proof of Theorem~\ref{thm:det}}\label{proof:det}
For the deterministic class, let $z_{s,j}\ge0$ be a feasible fractional hitting weight on $\Sset\times[d]$. Fixing $j$ gives $\sum_{s\in C_h}z_{s,j}\ge1$ for every $h$, so $\sum_s z_{s,j}\ge\tau^\star$. Summing over $j$ gives mass at least $d\tau^\star$. Conversely, copying an optimizer of \eqref{eq:tau} into every coordinate slice gives feasible mass $d\tau^\star$. Thus the product set system has fractional hitting number exactly $d\tau^\star$.

Normalize these weights to a distribution on $(s,j)$ and draw $M$ pairs independently, retaining their set of distinct values. Every $C_h\times\{J\}$ receives probability mass at least $1/(d\tau^\star)$, so its hit probability is at least
\[
1-\left(1-\frac1{d\tau^\star}\right)^M.
\]
For the reverse bound, linear-program duality supplies a distribution on $(h,J)$ such that every single pair $(s,j)$ belongs to $C_h\times\{J\}$ with probability at most $1/(d\tau^\star)$. For any fixed set $Q$ of at most $M$ pairs, the union bound gives hit probability at most $M/(d\tau^\star)$. Averaging over randomized $Q$ and truncating at one proves \eqref{eq:coveragebounds}.

\textbf{Deterministic lower bound.} For the regret lower bound, fix a policy and run it on the null transcript $Y_t=1/2$. Its audited state-coordinate support $Q$ is a random set of size at most $M$. The finite zero-sum game in \eqref{eq:coverage} has a dual prior on $(h,J)$ under which every distribution of $Q$ misses $C_h\times\{J\}$ with probability at least $1-c_M^{(d)}$. Draw $\theta$ independently and uniformly. On the miss event, every audited outcome equals $1/2$, so the true transcript agrees pathwise with the null transcript. Conditional on that transcript, the policy cannot depend on $\theta$. Averaging over $\theta$, every played action has target value $1/2$, while the optimal action has value $1/2+\gamma\rho$. Therefore
\[
\E[R_T]\ge\gamma\rho T\Pr\{Q\cap(C_h\times\{J\})=\varnothing\}
\ge\gamma\rho T(1-c_M^{(d)}).
\]

\textbf{Matching deterministic policy.} For the upper bound, draw a support $Q$ from an optimizer of \eqref{eq:coverage}. Test its pairs in any fixed order. To test $(s,i)$, repeatedly play coordinate $i$ with an independent uniform sign until state $s$ appears, then audit. The deterministic residual is nonzero exactly when $s\in C_h$ and $i=J$, in which case its sign together with the played sign reveals $\theta$. The policy then plays the optimal action forever. If every test is uninformative, it continues with independent uniform signs on an arbitrary coordinate. Before identification, uniform signs have expected target value $1/2$ and incur $\gamma\rho$ regret per round.

Let $L_Q$ be the total number of rounds required to complete all tests. Each waiting time has expectation at most $1/\mu_{\min}$, so $\E[L_Q]\le M/\mu_{\min}$. If $Q$ hits $C_h\times\{J\}$, all regret is charged within these test rounds. If it misses, at most $\gamma\rho T$ additional regret is charged. Hence, uniformly over $(h,J,\theta)$,
\[
\E[R_T]\le\gamma\rho\E[L_Q]
+\gamma\rho T\Pr\{Q\cap(C_h\times\{J\})=\varnothing\}
\le\frac{\gamma\rho M}{\mu_{\min}}
+\gamma\rho T(1-c_M^{(d)}).
\]
Dividing both bounds by $T$ and letting $T\to\infty$ proves the limit.\hfill$\square$

\subsection{Proof of Proposition~\ref{prop:actionstate}}\label{proof:actionstate}
\begin{proof}
The definitions give $0\le\alpha_{a,s}\le1$. For one stationary design round,
\[
\Pr\{X_t=a,S_t=s,B_t=1\}
=q_a\mu_a(s)\alpha_{a,s}
=\frac{\nu(a,s)}{\Lambda_{\rm ad}(\bm\nu)}.
\]
Summing over $(a,s)$ proves both the acceptance probability and the conditional accepted law. Independent fresh rounds therefore give i.i.d. accepted pairs.

For the causal lower bound, let $\mathcal F_{t-1}$ contain the history before $X_t$ is selected and set $N_a=\sum_{t\ge1}\ind\{\tau\ge t,X_t=a\}$. Conditional on $\mathcal F_{t-1}$ and $X_t=a$, the state law is $\mu_a$. Tonelli's theorem therefore gives
\[
\E\!\left[\sum_{t\ge1}\ind\{\tau\ge t,X_t=a,S_t=s\}\right]
=\mu_a(s)\E[N_a].
\]
The audited occurrence $(X_\tau,S_\tau)$ is one term in this sum and has probability $\nu(a,s)$. Hence $\nu(a,s)\le\mu_a(s)\E[N_a]$, so $\E[N_a]\ge m_a$. Summing over $a$ proves the waiting-time bound, and weighting by $\Delta_e(a)$ proves the sampling-regret bound. Under the stationary compiler, $\E[\tau]=\Lambda_{\rm ad}(\bm\nu)$ and $\E[N_a]=q_a\E[\tau]=m_a$, proving equality. Conditioning on the accepted history gives the stagewise predictable statement.
\end{proof}

\section{Supplementary theoretical results}\label{app:extensions}
This appendix gives the deterministic coverage specialization, the finite-summary counterexample, and the acquisition rule for known action-dependent diagnostic laws. These results refine the main guarantees without changing the action-independent assumptions used in the primary experiments. Their proofs are collected in Appendix~\ref{app:proofs}.

\subsection{Deterministic coverage and audit-only comparison}\label{app:deterministic}
The deterministic coordinate model has a sharper coefficient. Let $\mathcal Q_M=\{Q\subseteq\Sset\times[d]:|Q|\le M\}$. Define
\begin{equation}\label{eq:coverage}
c_M^{(d)}=\max_{\bm\omega\in\Delta(\mathcal Q_M)}\min_{h,J}\Pr_{Q\sim\bm\omega}\{Q\cap(C_h\times\{J\})\ne\varnothing\}.
\end{equation}
The product dictionary has fractional hitting number $d\tau^\star$ and
\begin{equation}\label{eq:coveragebounds}
1-\left(1-\frac1{d\tau^\star}\right)^M\le c_M^{(d)}
\le\min\left\{1,\frac{M}{d\tau^\star}\right\}.
\end{equation}

\noindent\begin{minipage}{\linewidth}
\begin{theorem}[Exact deterministic specialization]\label{thm:det}
Write the coordinate action as $X_t=(I_t,A_t)\in[d]\times\{\pm1\}$. For $0<\gamma\le1/2$, replacing \eqref{eq:generalmodel} by
$Y_t=\tfrac12+\gamma\theta A_t\ind\{I_t=J,S_t\in C_h\}$ gives
\[
\gamma\rho T(1-c_M^{(d)})\le R_T^\star\le
\gamma\rho T(1-c_M^{(d)})+\frac{\gamma\rho M}{\mu_{\min}}.
\]
For fixed $M$, $\mu$, and reliability class, $\lim_{T\to\infty}R_T^\star/T=\gamma\rho(1-c_M^{(d)})$.
\end{theorem}
\end{minipage}\par
See Appendix~\ref{proof:det} for the mixed-coverage duality argument and the matching audit policy.

For deterministic outcome contamination bounded by $\eps<\gamma/2$, threshold decoding remains valid and target comparison adds at most $2\eps T$. This does not cover overlapping informative and uninformative ranges. Sequential audit-only policies remain a necessary comparison because they use trusted outcomes without a reliability model \cite{cbm}. The pooled policy in Theorem~\ref{thm:diffuse} gives one such guarantee, and Section~\ref{sec:experiments} compares the conditional policies with marginal-dictionary and GP controls.

Standard GP information gain cannot determine the sharper class-dependent frontier. With the same rank-$d$ coordinate kernel, a known active coordinate gives $\mathfrak D_+=1/\Delta^2$, while an unknown coordinate gives $d/\Delta^2$; the kernel information-gain curve is unchanged. Extending the result to a full RKHS ball requires a separate packing or localization argument.

\subsection{Information absent from finite feedback summaries}\label{app:summarybarrier}
The next result allows bounded outcomes $Y\in[0,1]$ and a singleton diagnostic state. It therefore concerns general conditional feedback laws rather than only the Bernoulli family in \eqref{eq:generalmodel}.

\noindent\begin{minipage}{\linewidth}
\begin{proposition}[Finite-statistic distributional barrier]\label{prop:momentbarrier}
Fix any integer $p\ge1$ and bounded measurable statistics $\psi_1,\ldots,\psi_p:[0,1]\to\R$. For every integer $d\ge1$, $K=2^d$, $T\ge d$, and $0<\Delta\le1/4$, there are two finite verified-feedback dictionaries with identical target functions, the same rank-one centered target kernel, zero target misspecification, and identical conditional expectations of all $\psi_i(Y)$ for every environment and action. One admits a causal $M=d$ policy with $R_T=d\Delta$. In the other, every causal policy using at most $M$ audits satisfies
\[
\sup_e\E_e[R_T]\ge\frac{\Delta T}{2}e^{-M\kappa_\Delta},
\qquad
\kappa_\Delta=2\Delta\log\frac{1+2\Delta}{1-2\Delta}\le16\Delta^2.
\]
At $M=d$, $T=K$, and $\Delta=1/(4\sqrt K)$ with $K\ge2d$, the regret ratio is at least $K/(4\log_2K)$.
\end{proposition}
\end{minipage}\par
The construction and testing reduction are proved in Appendix~\ref{proof:momentbarrier}.

The construction combines Tverberg's convex-partition theorem \cite{baranysoberon} with active testing. It shows why misspecifying the conditional feedback law cannot be covered by the $2\eps T$ target-comparison term or by any fixed finite summary vector on an unrestricted law class: full-law distinguishability changes while every displayed statistic and target quantity remains fixed. It does not apply when the statistics identify a restricted conditional-law family; full laws or a characteristic representation can still distinguish the construction.

\subsection{Action-dependent diagnostic laws}\label{app:actionstate}
The structured model in Section~\ref{sec:problem} uses a common state law so that the reliable-set mass and the target reward separate. The rejection-sampling compiler has a broader scope. Suppose instead that, conditional on choosing action $a$ and on the past, the next state has a known law $\mu_a$ that is common across candidate environments. Let $\bm\nu\in\Delta(\A\times\Sset)$ be supported only on cells with $\mu_a(s)>0$, and define
\[
m_a=\max_{s:\mu_a(s)>0}\frac{\nu(a,s)}{\mu_a(s)},\qquad
\Lambda_{\rm ad}(\bm\nu)=\sum_am_a,\qquad
q_a=\frac{m_a}{\Lambda_{\rm ad}(\bm\nu)},\qquad
\alpha_{a,s}=\frac{\nu(a,s)}{m_a\mu_a(s)},
\]
where $\alpha_{a,s}$ is defined by the formula when $m_a\mu_a(s)>0$ and set to zero otherwise. Zero rows are never selected.

\noindent\begin{minipage}{\linewidth}
\begin{proposition}[Action-dependent compiler]\label{prop:actionstate}
Selecting $X_t\sim\bm q$ and auditing after state $s$ with probability $\alpha_{X_t,s}$ produces accepted action--state pairs that are i.i.d. $\bm\nu$, with acceptance probability $1/\Lambda_{\rm ad}(\bm\nu)$ per round. Moreover, every causal rule whose first audited pair has law $\bm\nu$ satisfies
\[
\E[N_a]\ge m_a,\qquad
\E[\tau]\ge\Lambda_{\rm ad}(\bm\nu),\qquad
\E_e\!\left[\sum_{t=1}^{\tau}\Delta_e(X_t)\right]
\ge C_e^{\rm ad}(\bm\nu):=\sum_a\Delta_e(a)m_a,
\]
where $N_a$ counts plays of action $a$ through the first audit time $\tau$. The stationary compiler attains all three bounds. The identities also hold conditionally for a predictable sequence of accepted-query laws, each held fixed until the next audit.
\end{proposition}
\end{minipage}\par

See Appendix~\ref{proof:actionstate} for the construction and causal lower bound.

This proposition changes only the acquisition compiler. If a separate exact finite dictionary supplies outcome laws $P_e^{a,s}$ and action-dependent targets $f_e(a)=\sum_s\mu_a(s)\E_e[Y\mid a,s]$, the accepted-history Hellinger argument in Proposition~\ref{prop:safeadaptive} carries over after replacing $\Lambda$ and $C_e$ by $\Lambda_{\rm ad}$ and $C_e^{\rm ad}$. Proposition~\ref{prop:actionstate} does not extend the reliable-set product factorization or adaptive converse, and it does not cover estimated state laws or a learned outcome dictionary.

\section{Experimental details and additional evidence}\label{app:experiments}
We first specify the synthetic comparisons and uncertainty calculations, then examine conditional-model errors and AI evaluation. Each study retains its full comparison family. The AI results distinguish offline identification, fitted-table policy evaluation, and independently tested synthetic policies. Numerical checks of the mathematical identities appear at the end of this appendix.

\subsection{Synthetic environments, policies, and statistical analysis}
Policy choices were fixed using separate development seeds. Implementation checks enforce action-before-state timing, posterior normalization, and the pathwise audit cap (Appendix~\ref{app:verification}). Each cell uses 256 paired seeds and horizon 8,192. The primary grid crosses three action geometries (rank-two and rank-eight signed coordinates and a periodic RBF kernel on 16 actions), reliability dictionary sizes one and four, known cold-state mass zero and $.75$, and budgets 128, 512, and 2,048. This gives 36 cells, seven methods, and 64,512 trajectories. All cells are retained. Reported cumulative regret is target regret over calendar time; offline simple regret is evaluated from the terminal recommendation. We also record audits, cheap state observations, completion time, and runtime.

\texttt{Pooled TS} uses the exact state-marginal signal dictionary and purchases every early outcome until the cap. \texttt{Conditional TS} maintains the joint posterior over $(h,z)$ and purchases every outcome in the union of candidate reliable sets until the cap; the audit rule does not learn. Both sample a target model for each action while budget remains, then exploit the maximum-likelihood model. \texttt{Direct UCB} purchases early outcomes and uses empirical-best exploitation after exhaustion. Joint and pooled likelihood controls use the same frozen action LP as Algorithm~\ref{alg:design}, but audit every early round. The known-$h$ design is an oracle. States are generated by the environment, never chosen by these controls.

The paired intervals are descriptive, unadjusted per-setting comparisons across 256 shared environment/noise streams. They do not give simultaneous 95\% coverage over the grid. Aggregate figure intervals use 2,000 paired bootstrap resamples within each cell, conditional on the fixed grid; cells are not resampled as random tasks. \texttt{Conditional TS} versus \texttt{pooled TS} changes both likelihood and audit gate; this is not an isolated kernel-sharing ablation. They do not establish the performance of learned audit allocation. The offline application replays evaluate correction estimators, while the JudgeBench and 32-judge policy studies below execute fitted finite-dictionary policies.

The overlapping suite replaces singleton reliable states by cyclic blocks of sizes two and four on eight states. It crosses two reward geometries, two state-arrival regimes, and three budgets (24 cells; 30,720 trajectories). Thirty-eight construction checks verify every target mean, overlap mass, posterior likelihood, and absence of unaudited outcomes. Paired 95\% intervals favor the conditional policy in 19 settings, are unresolved in five, and favor the pooled policy in none; its budget-wise geometric regret ratios are $.904,.522,.392$.

\begin{figure}[t]
\centering
\includegraphics[width=.92\linewidth]{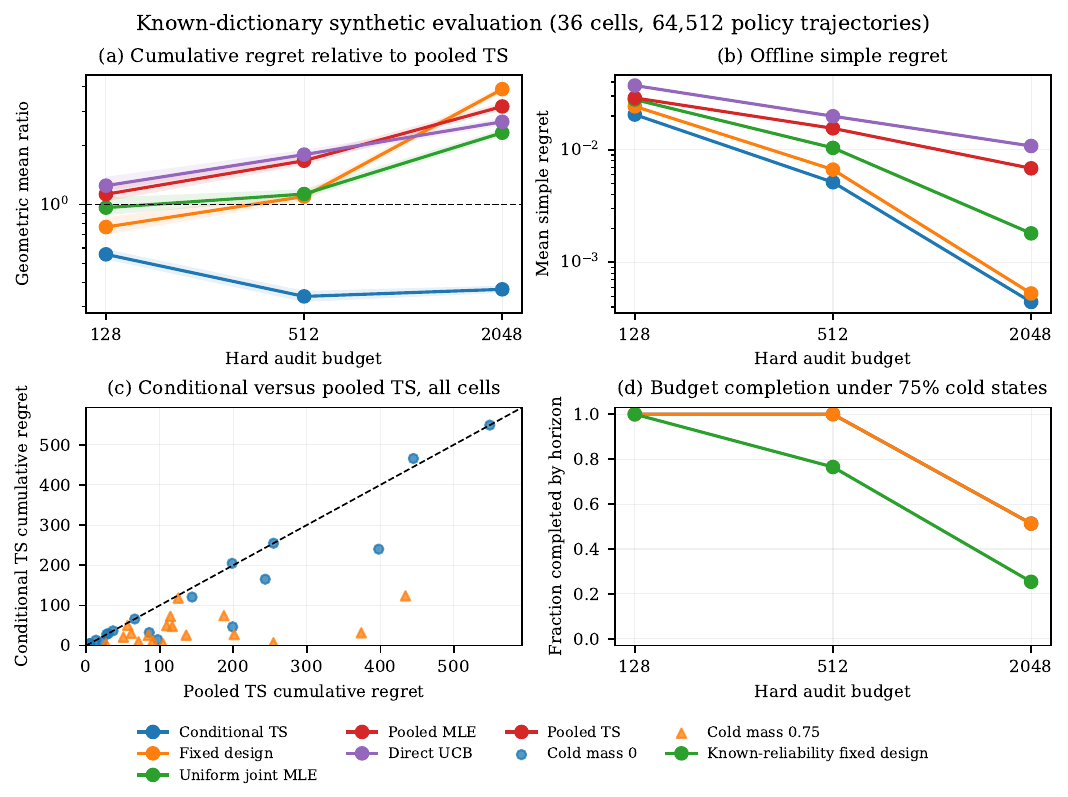}
\caption{Complete primary synthetic comparison. Panels (a) and (b) show cumulative-regret ratios and terminal simple regret. Panel (c) compares all 36 settings with \texttt{pooled TS}. Circles and triangles indicate candidate-state probabilities $1$ and $.25$, and points below the diagonal favor \texttt{conditional TS}. Panel (d) reports the fraction reaching the audit cap by the horizon when 75\% of states are uninformative. There are 256 paired replications per setting. Shading in (a) is a paired-bootstrap 95\% interval. The red curve denotes pooled maximum likelihood in (a) and pooled TS in (b).}
\label{fig:direct-full}
\end{figure}

\FloatBarrier
\subsection{Adaptive verification: regret and acquisition tradeoffs}
The theory-aligned adaptive suite uses eight actions, four squared-exponential kernel sections, and three prespecified reliability families: four singleton sets, six cyclic adjacent pairs, and four disjoint adjacent pairs. It crosses two horizon--cap pairs $(512,32)$ and $(2048,64)$, giving 24 settings. The floor $\eps=.25$, Chernoff alternative rule, analysis, and 512 paired replications were selected on a separate 24-system development study before the fresh suite was measured. Every environment, setting, and outcome is retained.

The fresh suite contains 36,864 trajectories. Safe adaptive versus the fixed robust joint design has geometric regret ratio $.933$ and nine lower-regret, 15 unresolved, and no higher-regret paired comparisons; the three reliability-family ratios are $.937,.934,.929$. Pure adaptive allocation has ratio $.915$ and 13 lower-regret and 11 unresolved comparisons, but does not preserve the robust information floor. One safe-adaptive point estimate is adverse, with ratio $1.021$, and remains unresolved. Independent reconstruction from every-round action counts and accepted success/failure counts reproduces total regret and every final likelihood for all trajectories; the maximum errors are $6.14\times10^{-12}$ and $8.53\times10^{-14}$.

A supplementary study was specified after the preceding outcomes and then run on new paired streams. It retains all 24 settings and matches the exact dictionary, zero-label calibration endowment, horizon, cap, and maximum-likelihood commit rule. Every trajectory has exactly $T$ action deployments and $M$ trusted labels. Regret charges all $T$ actions, including unaudited waiting rounds. Table~\ref{tab:safeadaptive} reports all comparators and separately measures the round on which the cap is reached. The audit-only controls either sample an environment from the exact posterior or choose the posterior-mean-best action, and purchase every early trusted outcome until the cap.

\begin{table}[!htbp]
\centering\footnotesize
\setlength{\tabcolsep}{3.2pt}
\begin{tabular}{@{}p{.31\linewidth}rrrrrr@{}}
\toprule
Comparator & \shortstack{Regret\\ratio} & \shortstack{Time\\ratio} & \shortstack{Error\\(pp)} & Lower & Unclear & Higher\\\midrule
Fixed robust joint design & .916 & 2.554 & $-7.06$ & 13 & 11 & 0\\
Pure adaptive design & 1.006 & .827 & $+0.66$ & 0 & 24 & 0\\
Product design & .916 & 4.764 & $-6.89$ & 12 & 12 & 0\\
Uniform action and verification & .907 & 4.764 & $-7.71$ & 11 & 13 & 0\\
Optimized action, uniform verification & .921 & 4.764 & $-6.96$ & 11 & 12 & 1\\
Audit-only dictionary TS & .943 & 4.764 & $-4.92$ & 6 & 18 & 0\\
Audit-only posterior greedy & .966 & 4.764 & $-4.21$ & 3 & 21 & 0\\
\bottomrule
\end{tabular}
\caption{Supplementary certified-adaptive comparisons on 512 paired replications per setting. Regret and time are safe/comparator geometric ratios across all 24 settings; time is the round of the $M$th audit. Error is the mean difference in final wrong-decision rate, in percentage points, so negative values favor safe adaptive. Lower, unclear, and higher count settings in which the paired 95\% interval places safe-adaptive regret below, across, or above the comparator. The intervals are unadjusted descriptive comparisons, so unclear does not establish equivalence. Comparator choices followed the fresh suite and are not independent confirmation.}
\label{tab:safeadaptive}
\end{table}

The one resolved adverse comparison is against optimized-action uniform verification for the disjoint-pair family at $(T,M)=(512,32)$ and the smoothest signal: ratio $1.078$ and paired regret difference $.738$ with interval $[.026,1.451]$. At the level of sample means, safe adaptive has higher regret in 2 settings against fixed robust design, 12 against pure adaptive, 1 against each product or uniform design, 3 against audit-only dictionary TS, and 5 against audit-only posterior greedy. Safe adaptive takes $4.764\times$ as many rounds geometrically as every-round controls (setting range $3.238$--$6.803$), so this is a trusted-label comparison rather than a general cost advantage. The supplementary study comprises 98,304 trajectories. Its acquisition accounting uses every action and every purchased label. Proposition~\ref{prop:safeadaptive}'s raw union-bound terms range from $14.47$ to $22.82$, so clipping makes the guaranteed error upper bound one in all 24 settings. The result proves a positive information floor but does not quantitatively explain the observed finite-budget gain. These exact-model experiments do not establish learned-dictionary validity, calibration coverage, a kernel effect, monetary savings, or real-task benefit.

\subsection{Gaussian-process and factorial controls}
A development screen for the Gaussian-process controls was specified before execution. It crossed coordinate and periodic-RBF actions, one and four reliability groups, cold-state masses zero and $.75$, budgets 64 and 256, horizon 1,024, and 24 repetitions. GP and identity controls received the same prior-scale grid, and UCB controls received the same exploration-width grid. Each method minimized mean cumulative regret over the complete development grid. The selected settings and the defaults were then frozen before an independently seeded measured comparison.

The measured grid adds budget 1,024, uses horizon 4,096 and 128 paired repetitions, and contains 24 cells, 13 methods, and 39,936 trajectories. \texttt{GP-UCB} and \texttt{GP-TS} use the target kernel and a Gaussian working likelihood with variance $1/4$ for bounded Bernoulli outcomes. \texttt{Identity-UCB} and \texttt{identity-TS} use the same updates and hyperparameters without cross-action covariance. The prior mean is $1/2$ and covariance is $(s_0\rho U)^2\bm K$, with $s_0\in\{.5,1,2\}$ in the screen. The UCB multiplier is chosen from $\{1,2,3\}$. Defaults use $s_0=1$ and multiplier two. Selection gives $(s_0,\beta)=(2,3)$ for \texttt{GP-UCB}, $s_0=2$ for \texttt{GP-TS}, $(.5,3)$ for \texttt{identity-UCB}, and $s_0=2$ for \texttt{identity-TS}. Both types audit every round until the cap and then exploit a posterior-mean maximizer. \texttt{Candidate-gated pooled TS} uses audited mean $1/2+u_z(a)/g$, where $g$ is the number of candidate reliability groups. \texttt{All-round pooled TS} instead uses $1/2+\rho u_z(a)$. These controls are computational GP baselines, not exact Bayesian inference for the finite signal dictionary. Every method observes 4,096 cheap states and respects the hard audit cap.

\begin{table}[!htbp]
\centering\small
\begin{tabular}{lrrrr}
\toprule
Policy & \shortstack{Cumulative\\regret $\downarrow$} & \shortstack{Standard\\error} & \shortstack{Simple\\regret $\downarrow$} & Audits\\\midrule
\texttt{Conditional TS} & \textbf{92.24} & 2.58 & \textbf{.0198} & 446.14\\
\texttt{Candidate-gated pooled TS} & 99.92 & 2.58 & .0219 & 446.14\\
\texttt{All-round joint TS} & 106.09 & 2.62 & .0241 & 448.00\\
\texttt{GP-UCB}, default & 119.27 & 2.64 & .0275 & 448.00\\
\texttt{GP-UCB}, selected & 122.34 & 2.65 & .0276 & 448.00\\
\texttt{Identity-UCB}, selected & 123.17 & 2.72 & .0284 & 448.00\\
\texttt{All-round pooled TS} & 124.30 & 2.66 & .0287 & 448.00\\
\texttt{Identity-UCB}, default & 124.83 & 2.78 & .0286 & 448.00\\
\texttt{GP-TS}, default & 130.40 & 2.65 & .0292 & 448.00\\
\texttt{GP-TS}, selected & 136.68 & 2.70 & .0310 & 448.00\\
\texttt{Reliability-Kernel-Design} & 138.46 & 2.74 & .0226 & 446.14\\
\texttt{Identity-TS}, default & 144.58 & 2.76 & .0325 & 448.00\\
\texttt{Identity-TS}, selected & 149.37 & 2.76 & .0337 & 448.00\\
\bottomrule
\end{tabular}
\caption{Independent Gaussian-process control suite. Means average the 24 fixed settings equally. SE is $[\sum_{c=1}^{24}s_c^2/128]^{1/2}/24$, where $s_c^2$ is the within-setting sample variance. Bold denotes the lowest mean regret. ``Selected'' denotes the setting chosen on the separate development grid.}
\label{tab:gpcontrols}
\end{table}

Against default \texttt{GP-UCB}, \texttt{conditional TS} has 15 lower-regret, seven unresolved, and two higher-regret paired comparisons and geometric-mean regret ratio $.663$. Against default \texttt{GP-TS} 15 comparisons favor the conditional policy and nine are unresolved, with ratio $.612$. Both \texttt{GP-UCB} losses occur in coordinate cells. The selected GP settings have higher aggregate regret than their defaults and are retained in Table~\ref{tab:gpcontrols}.

Matched comparisons separate the ingredients. Kernel versus identity covariance favors \texttt{GP-UCB} in three settings and \texttt{GP-TS} in six, with respective regret ratios $.925$ and $.886$. Joint versus pooled likelihood favors the joint model in six settings under the candidate gate and nine under all-round auditing, with ratios $.901$ and $.800$. Candidate-state auditing is favored in ten settings under either likelihood, with ratios $.820$ for the joint model and $.728$ for the pooled model. All remaining comparisons are unresolved, and none favor the comparator. Thus the complete conditional policy has a nominal advantage in these families, while the isolated kernel effect is smaller and often unresolved. The intervals are unadjusted descriptive 95\% normal intervals for paired seed differences within each cell.

\subsection{Feedback-law and target misspecification}\label{app:robustness}
The supported robustness grid has 36 cells, five methods, and 46,080 trajectories. Write the nominal mean as $1/2+\ind\{s\in C_h\}u_z(a)$. The feedback perturbation uses $1/2+[\ind\{s\in C_h\}(1-\alpha)+\ind\{s\notin C_h\}\alpha\rho/(1-\rho)]u_z(a)$, preserving the target exactly. Here $\rho=1/16$, $|u_z(a)|\le1/4$, and the conditional-error radius is $\alpha/4$. Exact flattening requires $\alpha=1-\rho$. The target perturbation instead adds a state-independent cosine residual $r(a)$, of radius $.05$ or $.10$, to both the conditional outcome mean and the scoring target. This tests target/dictionary error and also shifts the conditional means; it does not leave the full feedback model correct. On the finite full-rank RBF grid, it is not evidence that a specified RKHS norm bound is violated.

All reported perturbation experiments use the outcome laws above, with target rewards obtained by marginalizing those same laws. Policies, audit budgets, horizons, kernels, and perturbation magnitudes were specified before evaluation. Exact checks verified the marginal targets, bounded Bernoulli probabilities, and state independence in the flattened model. Table~\ref{tab:robustness} reports every specified perturbation, including settings in which conditional inference increases regret.
\begin{table}[!htbp]
\centering\small
\begin{tabular}{lrrr}
\toprule
Scenario & Conditional regret & Pooled regret & Geometric ratio\\\midrule
Nominal & 54.8 & 88.1 & .572\\
Feedback error $.10$ & 78.0 & 88.4 & .875\\
Feedback error $.20$ & 90.8 & 89.2 & 1.027\\
Feedback flattened & 92.4 & 87.4 & 1.068\\
Target perturbation $.05$ & 263.1 & 199.1 & 1.398\\
Target perturbation $.10$ & 350.7 & 204.5 & 1.799\\
\bottomrule
\end{tabular}
\caption{Robustness suite. Means are final cumulative target regret; ratios are geometric means over six cells per scenario.}
\label{tab:robustness}
\end{table}

\begin{figure}[!htbp]
\centering
\includegraphics[width=.88\linewidth]{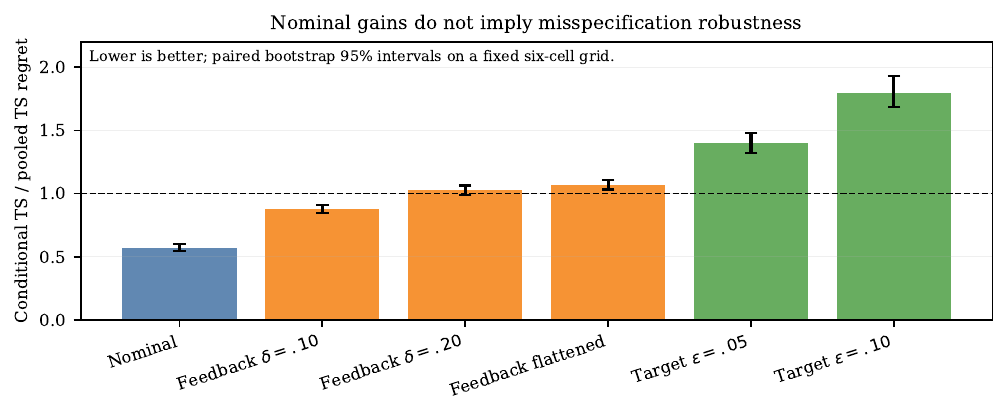}
\caption{Misspecification defeats the nominal advantage. Feedback perturbations preserve target rewards; target perturbations shift both rewards and outcome means. Bars show geometric mean ratios with paired-bootstrap 95\% intervals conditional on six cells.}
\label{fig:robustness}
\end{figure}

\subsection{Offline AI-system identification}\label{app:offline}
SAID RAG Evaluation contains 24 configurations and 15 requests per dataset, with three human ratings per output and four archived judge dimensions. The RAG method uses two-fold multivariate cross-fitting and a frozen additive generator--pipeline kernel. HotpotQA is development; PubMedQA remains unopened. AgentRewardBench development froze five binary judges and coefficient $.9$, the request split, audit design, 40,000-trial count, $.02$ simple-regret tolerance, and 90\% success endpoint. The confirmation program verified the frozen protocol hash before reading targets.

\textbf{Selection endpoint and evaluation costs.}
All selection experiments report the first audit budget attaining 90\% probability of simple regret at most $.02$. SAID uses 10,000 fixed-seed trials and AgentRewardBench uses 40,000. Audit-only methods receive the same trusted samples as corrected methods. Kernel smoothers are frozen before confirmation. Judge calls are counted even when archived; no missing ground truth is replaced by a prediction.

RAGTruth transfer provides no gain over its strongest kernel audit-only baselines: QA ties at 48 ratings, data-to-text ties at 6, and summarization uses 18 versus 12. SummEval's kernel baseline saturates the endpoint. No claim is made for the unevaluated PubMedQA split.

A separate 32-judge Civil Comments development screen is a negative control. Its best oracle proxy predicts only 1.07$\times$ mean audit savings, and no tested method reaches the endpoint by 160 trusted outcomes per judge. Its reserved 248-item split is excluded from the reported comparison.

For AgentRewardBench, the frozen judge set averages two priced GPT-4o-mini variants and three self-hosted Llama/Qwen variants. Logged API costs cover only the priced judges. The table therefore makes no monetary-saving claim. Raw or row-level derivatives are not redistributed.

\subsection{JudgeBench: conditional correction and dictionary fit}\label{app:judgebench}
\textbf{Conditional-correction policies.}
JudgeBench is analyzed only on the frozen calibration and development groups. The simulator draws states with calibration frequencies and scores each action by its development conditional means under that same state law. The retained data loaders skip the 106 reserved rows before decoding (see the provenance statement below). Policies, horizons, caps, repetitions, and seeds are unchanged from the frozen protocol; only the inconsistent regret estimand and acquisition ledger were repaired.

\begin{table}[!htbp]
\centering\small
\caption{JudgeBench conditional-correction development replay. Lower cumulative target regret is better. Calibration-informed methods share 144 offline labels and 1,584 archived judge outputs; audit-only receives neither. Negative paired differences favor state-aware audits; all 95\% intervals include zero. Bold marks the lowest mean, not significance.}
\label{tab:judgebench-all}
\setlength{\tabcolsep}{5pt}
\renewcommand{\arraystretch}{1.05}
\begin{tabular*}{\linewidth}{@{\extracolsep{\fill}}lrrrr@{}}
\toprule
& \multicolumn{2}{c}{$T=512$} & \multicolumn{2}{c}{$T=2048$}\\
\cmidrule(lr){2-3}\cmidrule(lr){4-5}
Method / comparison & $M=16$ & $M=64$ & $M=16$ & $M=64$\\
\midrule
\multicolumn{5}{@{}l}{\emph{Mean cumulative regret: shared calibration endowment}}\\
Calibration greedy & \textbf{0.000} & \textbf{0.000} & \textbf{0.000} & \textbf{0.000}\\
Calibration only & 9.327 & 9.242 & 37.190 & 37.015\\
Uniform audits & 9.090 & 8.252 & 35.885 & 33.249\\
State-aware audits & 9.040 & 8.566 & 36.315 & 34.649\\
\midrule
Audit only (no calibration) & 97.036 & 82.003 & 382.906 & 322.982\\
\midrule
\multicolumn{5}{@{}l}{\emph{Paired comparison: state-aware minus uniform audits}}\\
Mean difference & $-0.050$ & $0.314$ & $0.430$ & $1.400$\\
95\% interval & $[-0.305,0.189]$ & $[-0.039,0.691]$ & $[-0.518,1.360]$ & $[-0.147,3.020]$\\
\bottomrule
\end{tabular*}
\end{table}

\begin{table}[!htbp]
\centering\small
\caption{JudgeBench online acquisition. ``Reveals'' counts trusted-label requests. ``Unique'' counts distinct development items, and ``reuse'' is their difference. Values are Monte Carlo means. Offline calibration costs are additional.}
\label{tab:judgebench-acq}
\setlength{\tabcolsep}{5pt}
\begin{tabular}{@{}rrrrrrrr@{}}
\toprule
$T$ & $M$ & \multicolumn{3}{c}{State-aware audits} & \multicolumn{3}{c}{Uniform audits}\\
\cmidrule(lr){3-5}\cmidrule(lr){6-8}
 & & Reveals & Unique & Reuse & Reveals & Unique & Reuse\\\midrule
512 & 16 & 14.6 & 13.5 & 1.1 & 14.4 & 13.4 & 1.0\\
512 & 64 & 61.0 & 44.4 & 16.6 & 61.3 & 45.6 & 15.8\\
2048 & 16 & 14.5 & 13.5 & .9 & 14.3 & 13.4 & .9\\
2048 & 64 & 60.8 & 44.2 & 16.6 & 60.5 & 45.1 & 15.4\\\bottomrule
\end{tabular}
\end{table}

The policy Monte Carlo intervals condition on the fixed tables. Separately, 20,000 paired item bootstraps resample development items within each state while retaining calibration state weights. Calibration greedy is best or tied in 90.625\% of those resamples; its regret against the resampled best action has mean $.00211$ and percentile interval $[0,.03041]$. This is a finite-item sensitivity check, not population-level uncertainty or policy re-learning. Historical price, token, labeling-time, and latency fields are absent from the release.

\textbf{The finite-dictionary online policy.}
We separately execute the online policy in Algorithm~\ref{alg:design}. The reliability dictionary is constructed only from the 144 calibration items: deterministic two-fold validation selects two candidate sets for each prespecified cardinality $k\in\{1,2\}$, signals are refit on all calibration items, and the dictionary and both designs are then frozen. This construction uses 144 trusted labels and 1,584 archived judge outputs before online interaction. Dictionary learning is outside the fixed-dictionary theorem. The online policies share this endowment, observe one archived judge output per round, and obey the same pathwise trusted-label cap. All four $(T,M)$ cells, both dictionary sizes, matched uniform-state and uniform-action ablations, calibration controls, and audit-only TS were retained.

The development Bernoulli NLL is $.64475$ for the pooled calibration model and $.65311$ for a saturated action--state calibration model, versus $.68828$ and $.67361$ for the $k=1$ and $k=2$ dictionaries. Thus neither selected dictionary predicts the development table as well as the pooled model. Table~\ref{tab:judgebench-actual} shows the corresponding policy result.

\begin{table}[!htbp]
\centering\small
\caption{Non-confirmatory JudgeBench execution of Algorithm~\ref{alg:design}. RK denotes \texttt{Reliability-Kernel-Design}. Entries are mean cumulative regret across 512 repetitions. Lower is better. Both prespecified dictionary sizes and all four cells are retained.}
\label{tab:judgebench-actual}
\setlength{\tabcolsep}{4.0pt}
\begin{tabular}{rrrrrr}
\toprule
$T$ & $M$ & \shortstack{Calibration\\greedy} & \shortstack{Uniform\\verified TS} & \shortstack{RK design\\$k=1$} & \shortstack{RK design\\$k=2$} \\
\midrule
512 & 16 & 0.000 & 5.469 & 58.801 & 80.364 \\
512 & 64 & 0.000 & 4.921 & 62.836 & 113.122 \\
2048 & 16 & 0.000 & 21.758 & 232.605 & 288.367 \\
2048 & 64 & 0.000 & 19.090 & 238.288 & 312.919 \\
\bottomrule
\end{tabular}
\end{table}

The geometric-mean regret ratios of the RK designs to uniform verified TS are $11.63$ and $16.46$. Against the same action design with uniform-state auditing, the $k=1$ design is worse in all four cells; the $k=2$ design is better in one cell and worse in three. The selective gate approximately doubles the rounds needed to collect $M$ audits for $k=1$ and quadruples them for $k=2$. At $(T,M)=(2048,64)$, the two RK designs use 30.0 and 24.6 unique development items for 64 reveal calls, versus 43.6 under uniform verification. This is adverse, non-confirmatory development evidence: the 100 development labels had been opened by earlier analyses, while retained loaders skip the 106 reserved rows before decoding.

\subsection{Thirty-two LLM judges: signal clipping and reliability coverage}\label{app:panel32}
\textbf{Task, feedback, and calibration.}
A fresh replay couples all 32 archived LLM judges from Li et al.~\cite{panel32} to ChaosNLI human-majority labels~\cite{chaosnli}. An action chooses a judge before the task family is realized; the three action-independent states are MNLI-matched, SNLI, and $\alpha$NLI; the cheap output is the selected judge's archived class; and an audit reveals the majority of 100 human annotations. The hash split contains 591 calibration and 1,248 development items. The 1,140-row reserved partition is excluded from the reported study. All ten policies, both prespecified dictionary sizes, four horizon--budget settings, and 512 repetitions were fixed before outcome decoding for this study. Calibration-informed policies share an offline endowment of 591 trusted labels and 18,912 judge outputs, observe one archived output per online round, and obey the same pathwise reveal cap. The study comprises 20,480 trajectories.

The saturated action--state model has development Bernoulli NLL $.48280$, compared with $.52020$ for the pooled model and $.59460$ and $.53301$ for the $k=1$ and $k=2$ dictionaries. All 64 refit signal entries for $k=1$ and 63 of 64 for $k=2$ equal the clipping ceiling $.24$. Consequently, both action designs collapse to a one-point design chosen from a clipped tie. The RK policies have lower fixed-table regret than their uniform-action and uniformly verified controls, but their regret is exactly equal to the matched uniform-state ablation in every cell and repetition. The state gate therefore changes only acquisition time: at $(T,M)=(4096,128)$, the mean time to collect the cap is 192.5 rounds for $k=1$ and 386.5 for $k=2$, versus 128 under uniform-state verification. This adverse result shows that high-accuracy judge rewards can lie outside the useful range of the centered clipped dictionary. It supplies no positive evidence for state-aware auditing; historical price, token, latency, and human-labeling costs are unavailable and are not imputed.

\textbf{Separating reliability coverage from action-model expansion.}
Expanding both the reliability class and the signal dictionary changes two parts of the policy simultaneously. A postdiagnostic study separates their effects. For each $k\in\{1,2\}$, the reliability class contains all three $k$-subsets of the three task states, while the signal dictionary retains only the two baseline-offset signals selected by calibration-only cross-validation; the environment class is their Cartesian product. The reliability LP therefore gives a uniform state gate, while the action LP is unchanged from the two-signal design. Because the common fitted baseline cancels from every likelihood contrast, the product-separation guarantee applies conditional on these fitted signals. It does not cover calibration or population uncertainty.

\begin{table}[!htbp]
\centering\small
\caption{Thirty-two-judge development replay: cumulative regret, mean $\pm$ Monte Carlo standard error from 512 paired replications. Each method keeps the same two fitted signals and uniform state gate. The full-class method includes all reliability sets of size $k$ in its decoder, while the matched control includes only the two fitted sets. Bold denotes the lower mean in each row, not a significant difference.}
\label{tab:panel32decoupled}
\setlength{\tabcolsep}{10pt}
\begin{tabular}{@{}rrrr@{}}
\toprule
$T$ & $M$ & Full reliability class & Two fitted reliability sets\\\midrule
\multicolumn{4}{@{}l}{\textit{One reliable state ($k=1$)}}\\
1,024 & 32  & $32.20\pm.92$   & $\mathbf{31.79}\pm.93$\\
1,024 & 128 & $36.09\pm.82$   & $\mathbf{35.65}\pm.83$\\
4,096 & 32  & $127.10\pm3.79$ & $\mathbf{125.76}\pm3.79$\\
4,096 & 128 & $140.88\pm3.63$ & $\mathbf{139.90}\pm3.64$\\\midrule
\multicolumn{4}{@{}l}{\textit{Two reliable states ($k=2$)}}\\
1,024 & 32  & $\mathbf{16.00}\pm.24$ & $16.32\pm.24$\\
1,024 & 128 & $\mathbf{20.38}\pm.21$ & $20.57\pm.20$\\
4,096 & 32  & $\mathbf{61.38}\pm.98$ & $62.60\pm.97$\\
4,096 & 128 & $\mathbf{69.61}\pm.91$ & $70.03\pm.90$\\
\bottomrule
\end{tabular}
\end{table}

Closing the reliability class changes mean regret little relative to the matched two-set model: the geometric ratios are $1.011$ for $k=1$ (all four comparisons unresolved) and $.986$ for $k=2$ (one lower-regret and three unresolved comparisons). These intervals do not establish equivalence. Relative to full closure, which also inserts every fitted signal into the action LP, the decoupled design improves all four cells for each $k$, with ratios $.802$ and $.862$. Thus the earlier closure penalty came from expanded action-signal constraints, not uniform state coverage. The $k=2$ design also improves all four cells against a pure hitting gate (ratio $.829$), uniform verified TS ($.581$), and calibration greedy ($.464$). The adverse $k=1$ result is retained: it is worse than uniform verified TS in all four cells (ratio $1.143$) and unresolved against pure hitting ($.990$).

The study comprises 24,576 trajectories over 48 policy--cell groups. Independent checks verify the full reliability class, unchanged two-signal action design, hard caps, acquisition ledgers, and regret rescoring; the largest rescore error is $1.38\times10^{-11}$. Calibration-informed policies share 591 trusted labels and 18,912 archived judge outputs, whereas audit-only TS does not. Monetary, token, latency, and labeling-time prices are unavailable and are not imputed. Protected confirmation rows were skipped before decoding. This is development evidence selected after diagnosing the initial Panel32 failure, not fresh confirmation. It neither learns the audit gate nor establishes population-level coverage.

\textbf{Calibration uncertainty.} An aggregate-only calibration analysis sharpens this boundary. Simultaneous two-sided Clopper--Pearson intervals use 95\% familywise coverage over all 224 pooled and state-subset action means considered by the selection family. None of the four fitted targets certifies a unique best action: their point-estimate margins are $.00779,.000216,.00919,0$, whereas the corresponding lower-bound versus strongest-competitor-upper-bound margins are $-.108,-.156,-.0959,-.134$. A conservative nearest-candidate Hoeffding calculation, holding these point gaps fixed, requires respectively 300,053, 391,088,570, and 215,323 observations per controlled quantity; the zero-gap target cannot certify at any finite size under this rule. These are sufficient sizes, not power estimates or collection recommendations. Failure to certify does not show that a selected action is wrong, but it prevents this aggregate evidence from supplying a useful $\delta$ in Corollary~\ref{cor:caltransfer}; exact dictionary coverage is stronger still. The calculation assumes independent calibration requests within each action and subset and does not cover dictionary misspecification or population shift.

\textbf{Reserved-data provenance.} Reserved groups are excluded from the reported development fitting, policy selection, and performance evaluation, but their access histories differ. Retained JudgeBench loaders skip the 106 reserved rows before decoding. RAGTruth reserved labels were processed by upstream preparation and included in aggregate summaries, without use in the reported development comparisons. The LPHB confirmation rows were exposed and subsequently quarantined. SummEval reserved rows were decoded during processing but not used for development. These groups therefore do not constitute a uniformly untouched confirmation pool. No study above is reclassified as independent confirmation on the basis of those reserved groups.

\subsection{Numerical checks of the theoretical identities}\label{app:verification}
These checks assess the implementation of algebraic identities, optimization programs, and trajectory accounting. They supplement the proofs in Appendix~\ref{app:proofs} and are not an independent proof review or evidence of practical performance. Table~\ref{tab:numericalchecks} groups the checks by the mathematical claim they assess. Full generators, tolerances, and outputs are supplied with the reproducibility code.

\begin{table}[!htbp]
\centering\small
\caption{Numerical verification by scientific purpose. Errors are maximum absolute discrepancies unless the entry specifies a ratio or a violation count. Identity checks and synthetic theorem diagnostics are distinct from the policy comparisons reported in Section~\ref{sec:experiments}.}
\label{tab:numericalchecks}
\setlength{\tabcolsep}{5pt}
\begin{tabular}{@{}>{\raggedright\arraybackslash}p{.25\linewidth}>{\raggedright\arraybackslash}p{.32\linewidth}>{\raggedright\arraybackslash}p{.36\linewidth}@{}}
\toprule
Identity or guarantee & Evaluation systems & Observed discrepancy or check\\\midrule
Product separation & 180 random dictionaries; 40 coordinate cases & Minimum normalized margin $.9999999999999989$; coordinate error $1.74\times10^{-18}$\\[3pt]
Deterministic coverage & 1,440 LP cases from 240 set systems & Hitting error $8.88\times10^{-16}$; coverage violation $2.22\times10^{-16}$\\[3pt]
Bernoulli testing bounds & 2,000 separation cases; 100,000 KL and 251,001 Hellinger comparisons & No inequality violation; observed separation at least twice the conservative threshold\\[3pt]
Causal acquisition & 6,300 stationary programs; 5,000 occupancy programs; 45,157 adaptive histories & Acceptance-law error $3.34\times10^{-16}$; causal cost error $1.14\times10^{-13}$\\[3pt]
Finite-horizon accounting & 1,000 joint designs; independent forward recursion & Design-length error $4.27\times10^{-14}$; regret error $2.85\times10^{-14}$\\[3pt]
Adaptive information floor & 144 development systems and measured adaptive laws & Minimum information divided by $\eps I_0$ equals one\\[3pt]
Diffuse reliability & 44 hitting LPs; 210,000 KL comparisons; 19,200 decoder trajectories & Hitting error $8.88\times10^{-16}$; no five-standard-deviation upper-bound violation\\[3pt]
Misspecification bounds & 100,000 random systems; 10,000 flattening cases; 72,000 decoder trajectories & No constant violation; target-preservation error $1.11\times10^{-16}$\\[3pt]
Finite-summary construction & Tree depths through 12; moment orders through eight & 8,188 adaptive paths checked; zero requested-moment discrepancies on 29,520 points\\
\bottomrule
\end{tabular}
\end{table}

\textbf{Trajectory reconstruction.} The adaptive study records every action and every accepted outcome. Independent reconstruction of all 36,864 primary trajectories reproduces cumulative regret and terminal likelihoods to $6.14\times10^{-12}$ and $8.53\times10^{-14}$, respectively. The 98,304 supplementary trajectories use the same reconstruction. Timing checks ensure that actions precede the current state, audit decisions follow it, and every path respects the hard cap. These checks validate accounting rather than the accuracy of a learned outcome model.

\textbf{Optimization and acquisition checks.} Additional product-design comparisons verify the exact $\chi(\bm p)$ cost on 1,000 cases and the prior-mixture and fixed-design dual identities on 500 cases each. The largest prior-mixture error is $9.02\times10^{-11}$, and the largest fixed-design inequality violation is $1.19\times10^{-15}$. The finite-horizon design/exploitation partition has error $1.57\times10^{-13}$. These independent formulations detect coding or normalization errors without replacing the analytical arguments.

\FloatBarrier
\section{Detailed comparisons with prior work}\label{app:comparators}

\textbf{Verified and prediction-powered bandits.}
\texttt{PP-LUCB} studies best-arm identification with LLM judges, selective audits, and residual-aware allocation \cite{aolucb}. PPAI estimates prediction weights under adaptive sampling \cite{ppai}. Limited-verification bandits quantify the value of trusted observations under attacks \cite{rangi}. These works establish that auditing and correction are not new primitives. Our objective is cumulative true-reward regret under a hard audit budget, and our new parameter is an unknown conditional-reliability class.
\texttt{MLA-UCB} uses machine-learning-generated surrogate rewards for cumulative-regret reduction \cite{mlaucb}, while confidence--budget matching treats optional costly observations in sequential learning \cite{cbm}. They provide controls for judge-assisted and audit-scarce regimes, respectively. Neither models a latent post-action reliability set coupled to kernel action geometry.

\textbf{Kernel and multi-fidelity optimization.}
Auxiliary-feedback \texttt{IGP-UCB} uses parameterized correlated feedback \cite{verma}. \texttt{PA-GP-UCB} models prediction and reward jointly and observes both during online selection \cite{pagp}. Robust multi-fidelity BO guards against unreliable low-fidelity sources \cite{mikkola}. Kernel experimental design and misspecified GP bandits already control target approximation error \cite{bogunovic,camilleri}. ASB permits post-surrogate acquisition: its fixed-comparison theorem gives PRE/POST variance-optimal allocations, its main guarantees cover finite policy classes and a hard-budget linear construction, and its learned-refinement proposition charges residual-model error \cite{asb}. Chandak et al. optimize context-conditioned collection policies and use exact multi-proposal rejection sampling to reuse adaptively collected data \cite{chandak}. These results establish post-surrogate timing, residual-based allocation, kernel modeling, and rejection sampling as existing components.

\textbf{Controlled sensing and experimental design.}
Controlled multihypothesis testing optimizes Chernoff information over sensing controls, and adaptive causal controls can improve on open-loop designs with three or more hypotheses \cite{nitinawarat,mukherjee}. Bayesian decision-region determination already selects tests only to resolve the action-relevant equivalence class, including correlated noisy tests \cite{javdani,cheneced}. Transductive linear best-arm identification separates measurement and decision sets through experimental design \cite{fiez}; recent hybrid-feedback GLB identification optimizes a joint query space and includes heterogeneous acquisition costs \cite{zeng}. These methods supply decision-focused likelihood tests and query designs with exponentially decreasing error. Proposition~\ref{prop:safeadaptive} composes their predictable adaptive allocation principle with a fixed robust information floor and the exact action-before-state cost identity. The composition is scoped to this protocol and is not a new generic adaptive-testing principle. Our reliability--action factorization, adaptive converse, and deterministic coverage formula specialize this approach to unknown reliable-state sets; they do not characterize general adaptive optimality.

\textbf{Mixture testing and the diffuse regime.} Adaptive mixture testing already shows that repeated allocation can overcome rare-component dilution when the same latent item can be queried repeatedly \cite{leevaliant}. Our diffuse construction instead makes state repetitions negligible over the decision horizon. After the negligible-collision coupling, marginalizing the reliability set gives a standard $d$-armed Bernoulli identification problem with gap $\rho\Delta$; classical best-arm lower bounds and fixed-budget conversions already yield the $d/(\rho^2\Delta^2)$ scale up to logarithms \cite{garivierkaufmann,fc2fb}. The resulting comparison shows that rank and hitting number alone are insufficient to determine audit complexity.

\textbf{Partial monitoring.} Kernelized linear partial monitoring already supplies information-directed sequential policies, observability-based regret classifications, and RKHS observation operators \cite{kirschnerpm}. Its contextual extension also plans against a known context distribution, so a sequential information ratio, context-weighted observability coefficient, or kernelization is not new here. A finite verified-feedback dictionary can be unfolded into a constrained partial-monitoring game over contingent audit maps; our protocol-specific remainder is the hard pathwise audit cap together with the exact action-before-state compilation cost. This is why the paper states a structured finite-family factorization rather than a new generic sequential complexity.

\textbf{Nonparametric conditional feedback.} Nonparametric models provide another way to represent conditional feedback laws. \texttt{CME-UCB} already estimates conditional mean embeddings while sequentially optimizing an unknown conditional expectation \cite{raychowdhurycme}; distributionally-aware kernel bandits model probability kernels through kernel mean embeddings and prove UCB guarantees and lower bounds \cite{takemorikme}; and active cost-aware streaming labeling chooses expensive labels after inputs arrive and gives matching discrete and continuous-RKHS rates \cite{caikandasamy}. These models do not impose the action-before-state protocol with a hard audit cap. Extending their estimators to that protocol would require accounting for the acquisition costs in Proposition~\ref{prop:jointdesign}.

\textbf{Post-action contexts.} Post-serving-context bandits also observe informative features only after action selection; \texttt{poLinUCB} learns a pre-to-post context map and attains tight linear-contextual regret \cite{wangpost}. The trusted outcome here is optional and subject to a hard pathwise budget. The row-max formula is the exact domination constant obtained when classical rejection sampling is restricted to proposals of the form $q(a)\mu(s)$; it quantifies acquisition cost under this proposal restriction.

\textbf{Cost-sensitive acquisition.} Costly-observation bandits provide complementary cost-sensitive guarantees. Schur et al. give two-phase GP algorithms, cost-adjusted information-gain complexities, and matching finite-arm rates for observation-priced objectives \cite{schur}. Gong and Zhang establish hard free-observation-budget phase transitions for contextual bandits with paid observations \cite{gong}. Harding et al. choose a costly source before an unobserved group is revealed and obtain matching budgeted minimax rates by effective-sample-size design and post-stratification \cite{harding}; Reyzin et al. give tight cost-adaptive rates for heterogeneous expert queries \cite{reyzin}. Progressive-feedback bandits quantify early-surrogate value by conditional mutual information, but eventually observe the primary outcome for every play \cite{impatient}. Cost-aware acquisition, hard feedback budgets, and rejection-sampled reuse are therefore prior methods. Harding et al. do not select labels after seeing the group or study sequential decision regret, while Chandak et al. do not study hard audit budgets, sampling regret, or latent reliability. The remaining scoped distinction is the finite-dictionary reliability--kernel theorem and its adaptive converse; Schur et al. explicitly leave general-kernel optimality open, so Proposition~\ref{prop:jointdesign} is not presented as a matching kernel converse.

\textbf{LLM routing with restricted feedback.} Recent LLM-routing studies consider related feedback restrictions. Wei et al. give linear-contextual upper and lower bounds under a limited reward-feedback budget \cite{weillm}. Sridhar et al. combine context-dependent graph feedback with ML surrogate rewards and use expert aggregation to recover a reward-only guarantee when those surrogates are misspecified \cite{cabs}. Unlike both settings, trusted outcomes here are optional after observing a post-action state, and the unknown object is the state-conditional reliability law.

\textbf{AI evaluation.}
ARES uses prediction-powered inference for RAG evaluation \cite{ares}. HbBoPs applies deep-kernel Bayesian optimization to prompt selection \cite{hbbops}. NAOD designs preference queries while accounting for judge nuisance \cite{naod}. Our model is narrower: fixed configurations, i.i.d. requests, bounded outcomes, and a fixed kernel. It excludes contextual routing, nonstationarity, and within-agent-step credit assignment.

\begin{table}[!htbp]
\centering\footnotesize
\begin{tabular}{>{\raggedright\arraybackslash}p{.22\linewidth}>{\raggedright\arraybackslash}p{.31\linewidth}>{\raggedright\arraybackslash}p{.37\linewidth}}
\toprule
Comparator & Established capability & Distinction tested here\\\midrule
ASB \cite{asb} & post-surrogate acquisition, learned residual error, hard budgets & latent reliability dictionary and kernel decision geometry\\
Multi-rejection \cite{chandak} & exact reuse under multiple collection policies & hard-budget regret and latent reliability learning\\
Cost-aware GP \cite{schur} & optional labels and cost-adjusted kernel information gain & post-state reliability and hard pathwise audits\\
Controlled sensing \cite{nitinawarat,mukherjee} & adaptive control and likelihood testing & latent reliability cost under action-before-state timing\\
Auxiliary \texttt{IGP-UCB} \cite{verma} & kernelized parameterized auxiliary feedback & audit scarcity and latent reliable-state dictionary\\
\texttt{PP-LUCB} \cite{aolucb} & selective audits for identification & cumulative regret and kernel action geometry\\
Robust / misspecified BAI \cite{altschuler,reda} & contaminated or structurally misspecified identification & separate target and audit-law error costs\\
Audit-only / CBM \cite{cbm} & valid without an informative judge & benchmark that proposed feedback must improve\\
\bottomrule
\end{tabular}
\caption{Comparison of feedback protocols, guarantees, and modeling assumptions.}
\label{tab:boundary}
\end{table}

\end{document}